\documentclass[]{article}
\newcommand{\rom}[1]{\uppercase\expandafter{\romannumeral #1\relax}}

\usepackage{graphicx}
\usepackage{caption}
\usepackage{amsfonts}
\usepackage{amsmath}
\usepackage{amssymb}
\usepackage{fancyhdr}
\usepackage{titlesec}
\usepackage{indentfirst}
\usepackage{booktabs}
\usepackage{verbatim}
\usepackage{color}
\usepackage{tcolorbox}
\usepackage{mathtools}
\usepackage{bm}
\usepackage{amsthm}
\usepackage{wasysym}
\usepackage{empheq}
\usepackage[font=small,labelfont=bf]{caption} 
\usepackage{chngcntr}
\usepackage[page,header]{appendix}
\usepackage{subfig}
\usepackage{titletoc}
\usepackage{dsfont}
\usepackage{algorithm}
\usepackage{algorithmic}
\numberwithin{equation}{section}
\newtheorem{theorem}{Theorem}[section]
\newtheorem{lemma}[theorem]{Lemma}

\newtheorem{proposition}[theorem]{Proposition}

\newtheorem{remark}[theorem]{Remark}

\def\thetheorem {{\arabic{section}.\arabic{theorem}}}
\newcommand{\argmin}{\operatornamewithlimits{argmin}}
\newcommand{\argmax}{\operatornamewithlimits{argmax}}
\usepackage{tikz}

\def\a{\alpha}

\def\b{\beta}

\def\e{\epsilon}
\def\g{\gamma}
\def\lam{\lambda}
\def\o{\omega}

\def\s{\sigma}
\def\th{\theta}

\def\th{\theta}

\def\D{\Delta}

\def\A{\mathbb A}

\def\E{\mathbb E}

\def\R{\mathbb R}
\def\S{\mathbb S}

\def\mM{\mathcal M}
\def\mN{\mathcal N}

\def\l{\left}
\def\r{\right}
\def\la{\left\langle}
\def\ra{\right\rangle}
\def\ll{\left\lVert}
\def\rl{\right\rVert}
\def\lv{\left\lvert}
\def\rv{\right\rvert}

\def\pt{\partial}
\def\nb{\nabla}

\def\ds{\displaystyle}

\def\h{\hat}
\def\t{\tilde}

\def\Vinf{V^\infty}
\def\piinf{\pi^\infty}
\def\eps{\e}
\def\H{\mathcal{H}}
\def\tcb{\textcolor{black}}
\title{Variational Actor-Critic Algorithms}
\author{Yuhua Zhu\thanks{Department of Mathematics and Halicioğlu Data Science Institute, University of California, San Diego, La Jolla, California, U.S.A; e-mail: yuz244@ucsd.edu}
\and Lexing Ying \thanks{Department of Mathematics, Stanford University, Stanford, California, U.S.A; e-mail: lexing@stanford.edu}}

\date{}

\begin{document}

%\begin{resume} ... \end{resume}
%\subjclass{90C40,93E20.}
%\keywords{Markov Decision Process, Reinforcement learning, Policy gradient, Optimal control}
%
\maketitle

%%-----------------------------
%%      your text
%%-----------------------------
\begin{abstract}
  We introduce a class of variational actor-critic algorithms based on a variational formulation
  over both the value function and the policy. The objective function of the variational formulation
  consists of two parts: one for maximizing the value function and the other for minimizing the
  Bellman residual. Besides the vanilla gradient descent with both the value function and
  the policy updates, we propose two variants, the clipping method and the flipping method, in order to
  speed up the convergence. We also prove that, when the prefactor of the Bellman residual is
  sufficiently large, the fixed point of the algorithm is close to the optimal policy.
\end{abstract}
\section{Introduction}

%\LY{check the logic and text of this whole section.}
%In model-free Reinforcement Learning (RL), the transition dynamics $P$ is typically unknown.  

Consider a discounted Markov Decision Process (MDP) $\mM = (\S,\A, P, r, \g)$. Here $\S$ is the
state space and $\A$ is the action space. $\D(\S)$ and $\D(\A)$ denote the set of probability
distributions over $\S$ and $\A$, respectively. $P: \S\times\A \to \D(\S)$ is the transition kernel,
$r: \S\times \A \to \R$ is the reward function, and $\g\in(0,1)$ is the discounted factor. For each
state-action pair $(s,a)$, we denote by $P(s'|s,a)$ the transition probability from state $s$ to
state $s'$ given action $a$, $r(s,a)$ the immediate reward received at state $s$ with action $a$.

A policy $\pi:\S\to\D(\A)$ represents an action selection rule, where $\pi(a|s)$ specifies the
probability of taking action $a$ at state $s$.  The state value function $V^\pi(s)$ is the expected
discounted cumulative reward if one starts from an initial state $s$ and follows a policy $\pi$ with
step $t=0,1,\ldots$:
\begin{equation}\label{eq: def of Vpi}
  V^\pi(s) = \underset{\substack{a_t\sim\pi(\cdot|s_t)\\s_{t+1}\sim P(\cdot|s_t,a_t)}}{\E}\l[\sum_{t\geq0} \g^tr(s_t,a_t) | s_0 = s\r].
\end{equation}
The value function $V^\pi(s)$ also satisfies the Bellman equation \cite{sutton2018reinforcement},
\[
V^\pi(s) = \underset{\substack{a_t\sim\pi(\cdot|s_t)\\s_{t+1}\sim P(\cdot|s_t,a_t)}}{\E}
[r(s_t,a_t) + \g V^\pi(s_{t+1}) | s_t = s].
\]
The state-action value function $Q^\pi(s,a)$, often referred as the $Q$-function, is the expected
discounted cumulative reward if one takes action $a$ at initial state $s$: $Q^\pi(s,a) =
\E[\sum_{t\geq0} \g^tr(s_t,a_t) | s_0 = s, a_0=a]$. The two functions $V^\pi$ and
$Q^\pi$ are related in the sense that $V^\pi(s) = \E_{a\sim\pi(\cdot|s)} Q^\pi(s,a)$.

A primary goal of reinforcement learning (RL) is to learn the optimal policy $\pi^*$ and its
corresponding value function $V^*$. Among various approaches, the policy gradient methods have
experienced significant advances recently, for example see
\cite{williams1992simple,kakade2001natural,peters2008natural,konda2000actor,schulman2015trust,schulman2017proximal}.
From the optimization perspective, a policy gradient method optimizes the following
objective over policy $\pi$ with gradient updates
\begin{equation}\label{eq:pg}
  \begin{aligned}
    \max_\pi\quad &\E_{s\sim\rho} V^\pi(s)\\ \text{s.t.}\quad &V^\pi(s) =
    \underset{\substack{a_t\sim\pi(\cdot|s_t)\\s_{t+1}\sim P(\cdot|s_t,a_t)}}{\E} [r(s_t,a_t) + \g
      V^\pi(s_{t+1}) | s_t = s],
  \end{aligned}
\end{equation}
where $\rho(s)$ is a positive probability distribution. Policy gradient methods are often more
convenient than the value based methods in the settings of continuous action space, high dimensional
action space, and partially observed MDP \cite{degris2012model, silver2014deterministic,
  peters2006policy}.  It is also quite flexible to adopt various kinds of policy parameterizations
in the policy gradient methods, which makes them powerful for both stochastic policies
\cite{baxter2001infinite, sutton1999policy} and deterministic policies
\cite{silver2014deterministic, lillicrap2015continuous}.

For policy gradient methods, entropy regularization is often included because it improves
exploration by discouraging premature convergence to suboptimal deterministic policies \cite{peters2010relative, mnih2016asynchronous,williams1991function}.  More
specifically, entropy regularization takes for example the following regularized maximization
formulation:
\begin{equation}\label{eq: goal}
  \begin{aligned}
    \max_\pi\quad &\E_{s\sim\rho} V_\lam^\pi(s) \\
    \text{s.t.}\quad &V_\lam^\pi(s) = \underset{\substack{a_t\sim\pi(\cdot|s_t)\\s_{t+1}\sim P(\cdot|s_t,a_t)}}{\E} [r(s_t,a_t) - \lam \log\pi(a_t|s_t) + \g V^\pi_\lam(s_{t+1}) | s_t = s].
  \end{aligned}    
\end{equation}
Let $\pi^*_\lam$ be the regularized optimal policy for \eqref{eq: goal}. Note that $\pi^*_0 = \pi^*$
(the non-regularized optimal policy) when $\lambda=0$ but $\pi^*_\lam$ is different from $\pi^*$
when $\lambda>0$. In what follows, we shall abbreviate the optimal value functions
$V^{\pi^*_\lam}_\lam$ and $ V^{\pi^*}$ as $V_\lam^*$ and $V^*$, respectively.

%Note that different from $\pi^*_\lam$ with $\lam>0$, $\pi^*$ is the true optimal policy.

%The difference between $\pi^*_\lam$ and $\pi^*$ can be bounded by \[ \E_{s\sim\rho} V^{\pi^*_\lam}
%\leq \E_{s\sim\rho} V^* \leq \E_{s\sim\rho} V^{\pi^*_\lam} + \lam\log(|\A|).\] Here
%$V^{\pi^*_\lam}$ is the value function defined in \eqref{eq: def of Vpi}.The left inequality is due
%to the definition of $\E_{s\sim\rho} V^* = \max_\pi \E_{s\sim\rho} V^\pi$. The right one comes from
%$\E_{s\sim\rho} V^* \leq \E_{s\sim\rho} V^{\pi^*}_\lam \leq \E_{s\sim\rho} V^*_\lam =
%\E_{s\sim\rho} V^{\pi^*_\lam}-\lam \H(\pi) \leq \E_{s\sim\rho} V^{\pi^*_\lam} + \lam \log(|\A|)$.

%When $\lam$ is small, the regularized optimal policy is a good approximation of the non-regularized
%optimal policy in the sense that, \[ \rho^\top V^{\pi^*_\lam} \leq \rho^\top V^* \leq \rho^\top
%V^{\pi^*_\lam} + \lam\log(|\A|),\]where $ V^{\pi^*_\lam}$ is the non-regularized value function
%under the policy $\pi^*_\lam$ as defined in \eqref{eq: def of Vpi}.

%Without the regularizer $\H(\pi)$, it is usually hard for policy gradient to escape from local
%minimum, but with the regularizer $\H(\pi)$, the regularized optimal policy $\pi^*_\lam$ is
%different from the true optimal policy. So there is a trade-off between $\lam = 0$ and $\lam>0$. We
%later show in the numerical experiments that the stochastic algorithm we proposed can converge to
%the non-regularized optimal policy with high probability with the regularized objective function.

The most direct way of solving the optimization problem \eqref{eq: goal} is to update the policy $\pi$
according to the gradient $\nb_\pi \E_{s\sim\rho} V_\lam^{\pi}(s)$.  The calculation of $\nb_\pi
\E_{s\sim\rho} V_\lam^{\pi}(s)$ however involves computing the exact value function $V_\lam^\pi(s)$
or $Q_\lam^\pi(s,a)$ under the current policy $\pi$. With an accurate approximation of the value
function $V_\lam^\pi(s)$, this gradient-based method can achieve a linear convergence rate
\cite{agarwal2020optimality,mei2020global,cen2020fast}.  However, the calculation of value function
$V_\lam^\pi(s)$ can be computationally intensive for large MDP problems. Especially in the
model-free setting, a large data set is often needed in order to achieve a good approximation
\cite{munos2008finite,liu2020finite}.

To avoid the explicit computation of $V_\lam^\pi(s)$, the actor-critic methods
\cite{konda2000actor} have been widely studied in the literature
\cite{mnih2016asynchronous,wang2016sample,wu2017scalable,haarnoja2018soft,fujimoto2018addressing} as
a way to update the policy and value function at the same time. \tcb{However, the convergence of the actor-critic algorithm is guaranteed only for two-timescale algorithms \cite{yang2019provably}, where a smaller stepsize is used for the actor updates and a larger stepsize is used for the critic updates. }
\tcb{The stabilities of the actor-critic algorithms are often sensitive to the choice of stepsizes \cite{islam2017reproducibility}. }

%However due to the lack of an explicit objective function, the convergence analysis of the actor-critic algorithms remains an open problem for the nonlinear function approximation setting \cite{ konda2003onactor, yang2019provably}. 

%\LY{not sure we need the following sentence.}  Due to the asynchrony between actor and critic, the instability often arises in applications \cite{islam2017reproducibility}. Moreover,

{\bf Contributions.}
In this paper, we propose a new actor-critic method based on a variational formulation over the
policy and the value function. %where the algorithm naturally updates the policy and value function. 
Consider the optimization problem,
\begin{equation} \label{eq: obj1}
  \min_{V,\pi} \quad E(V,\pi)= \E_{s\sim\rho(s)} \l[ -V(s) + \frac{\beta}{2}\l(V(s)- E[r(s_t,a_t) +\g V(s_{t+1}) -\lam\log(\pi(a_t|s_t))|s_t = s]  \r)^2 \r],
\end{equation}
where $\rho\in\R^{|\S|}$ can be any positive probability distribution and $\b>0$ is a positive constant. The objective function \eqref{eq: obj1} consists two
parts: the first is to maximize the value function, while the second is to minimize the Bellman residual. The variational structure ensures that the vanilla gradient descent almost surely converges to a local minimum \tcb{without requirements on different stepsizes for $V$ and $\pi$ updates.}

\tcb{Besides, we pointed out that the vanilla gradient descent will lead to a direction increases $V^\pi_\lam$ at the initial stage because of the negative Bellman residuals.}
In order to improve the convergence speed of the vanilla gradient descent of \eqref{eq: obj1},
we further propose two variants. The first {\it clipping} method can be viewed as the gradient descent of
the objective function with a non-Euclidean metric. The second {\it flipping} method further
accelerates the convergence by continuously maximizing the value function in the right direction.

We prove that, when $\lam = 0$ and the prefactor $\b$ is sufficiently large, the fixed point of the
proposed algorithm is exactly the optimal policy $\pi^*$.  Furthermore, we prove that when $\lam>0$
i.e., in the regularized setting, the fixed point is close to the non-regularized optimal policy
$\pi^*$ for large $\b$ and small $\lam$.

%It is also verified in the numerical experiments that for natural policy gradient with almost exact
%value evaluation, it always converges to the regularized optimal policy, while for our method, it
%converges to the non-regularized optimal policy, which is more favorable.

%\LY{merge the below paragraph into the above one if necessary}The algorithm we propose is based on
%an explicit objective function, and it updates the value function and policy at the same time.

 %\LY{describe the sections briefly.}

{\bf Contents.}
The variational actor-critic algorithm is introduced in Section \ref{sec: algorithm}, where the
clipping and the flipping methods are first presented in the model-based setting first (Section
\ref{sec:modelbased}) and then in the model-free setting (Section \ref{sec: stochastic algo}). In
Section \ref{sec: thm}, we study the fixed point of the algorithm for both non-regularized (Section
\ref{sec: non_reg}) and the regularized (Section \ref{sec: reg}) objective functions. Several
numerical experiments are reported in Section \ref{sec: numerics} to demonstrate the performance of
the proposed algorithms.

%related work The objective function of our method essentially is the same as the AlgaeDICE method
%\cite{nachum2019algaedice}. There are three differences compared with their paper.  There are three
%main contributions of this paper. First, they propose the objective function with the motivation of
%automatically getting the on-policy gradient using off-policy data, while our motivation is to have
%a more efficient policy gradient method.

%They mentioned in the Appendix that they use similar trick in the numerical experiments, but our
%version is better than theirs (see Remark \ref{rmk: ALGAE}) and we make it clear why this could
%possibly improve the policy gradient largely.

%==============================================================================================================
\section{Variational Actor-Critic} \label{sec: algorithm}

Section \ref{sec:modelbased} presents the variational actor-critic algorithm in the model-based
setting, where the value function $V(s)$ is of the tabular form and the policy is parameterized with
the soft-max function. In Section \ref{sec: stochastic algo}, we introduce the stochastic variational 
actor-critic algorithm in the model-free setting, which applies to the general case of nonlinear 
approximation to the policies and $Q$-functions.

%Later, we will study a more general setting with continuous state space.
\subsection{Model-based setting} \label{sec:modelbased}
To simplify the discussions, we assume that both the state and action spaces are finite
discrete sets.
Consider the following minimization problem:
\begin{equation}\label{eq: objobj}
  \min_{V\in\R^{|\S|},\th\in\R^{|\S|\times|\A|}} \quad E(V,\pi)= -\rho^\top V + \frac{\beta}{2}\ll (I-\g P^{\pi_\th})V - r^{\pi_\th} + \lam \H(\pi_\th)\rl^2_\rho,
\end{equation}
where $\rho\in\R^{|\S|}$ is a positive probability distribution, and $\beta>0$ is a prefactor. The norm $\ll \cdot \rl_\rho$ is defined by $\ll x \rl^2_\rho:=
\sum_{s\in\S}x^2_s\rho_s$.  Here
$V=(V_s)_{s\in\S}$ in \eqref{eq: objobj} is an $|\S|$-dimensional vector, and the policy $\pi_\th=\{(\pi_\th)_{sa}\}_{s\in\S,
  a\in\A}$ is an $|\S|\times|\A|$ matrix. We assume the policy $\pi_\th$ is a soft-max function, i.e., for any pair $(s,a)\in\S\times \A$,
\[
(\pi_\th)_{sa} = \frac{e^{\th_{sa}}}{ \sum_{b\in\A} e^{\th_{sb}} }, \quad s\in\S, a\in\A. 
\]
Hereafter, we omit the subscript $\th$ of $\pi_\th$ for simplicity. The vector $r^\pi = (r^\pi_{s})_{s\in\S} \in \R^{|\S|}$ in \eqref{eq: objobj} is the reward under policy $\pi$ with the component $r^\pi_{s} = \sum_a r_{sa}\pi_{sa}$, where $r_{sa}$ is the immediate reward at $(s,a)$.   The matrix $P^\pi\in
\R^{|\S|\times|\S|}$ in \eqref{eq: objobj} is the transition matrix under policy $\pi$ with the
entry $P^\pi_{st}=\sum_{a}\pi_{sa}P^a_{st}$, where for each $a$, the matrix
$P^a\in\R^{|\S|\times|\S|}$ is the state transition matrix under action $a$. The vector
$\H(\pi) = (\H(\pi_s))_{s\in\S}\in\R^{|\S|}$ in \eqref{eq: objobj} is the entropy regularizer with the component $\H(\pi_s) = \sum_{a}\pi_{sa}\log\pi_{sa}$.

The minimization problem \eqref{eq: objobj} is a relaxation of the maximization problem of
$V^\pi_\lam$ as in \eqref{eq: goal}. Note that minimizing the first term $-\rho^\top V$ of the RHS
of \eqref{eq: objobj} has the same effect as maximizing $V$. On the other hand, the second term
$\ll(I-\g P^{\pi})V - r^\pi + \lam \H(\pi)\rl^2_\rho$ of the RHS of \eqref{eq: objobj} is the
norm of the Bellman residual. As $V^\pi_\lam = (I-\g P^{\pi})^{-1}(r^\pi-\lam\H(\pi_\th))$, the
minimization of the second term of the RHS of \eqref{eq: objobj} leads $V$ to the value function
$V^\pi_\lam$. Thus, combining the two terms of the RHS of \eqref{eq: objobj} yields that \eqref{eq:
  objobj} maximizes the true value function $V^\pi_\lam$ as in \eqref{eq: goal}.

One approach for solving the minimization problem \eqref{eq: objobj} is to update the $(V,\th)$ pair
following the gradients of the objective function. The gradients are given by
\begin{equation}\label{eq: obj}
  \begin{aligned}
    & \pt_{V_s}E          &=& -\rho_s +\beta(\ell_s\rho_s-\g \sum_tP_{ts}^\pi \ell_t\rho_t))                     &:=& G_{V_s}, \\
    & (F^+\nb_\th E)_{sa} &=&  \rho_s\beta \ell_s \l[- \g \sum_tP^a_{st}V_t - r_{sa} +\lam\log\pi_{sa}\r] + c_s &:=& G^{(0)}_{\th_{sa}},
  \end{aligned}    
\end{equation}
where $\ell(V,\pi)$ is an $|\S|$-dimensional function denoting the Bellman residual
\begin{equation}\label{def of L}
  \ell(V,\pi) = (I - \g P^\pi) V - r^\pi + \lam \H(\pi).
\end{equation}
Here the natural gradient is used for the policy updates in \eqref{eq: obj} with the Fisher
information matrix $F = \E_{s\sim \rho, a\sim\pi} \l[ (\nb_\th \log \pi)(\nb_\th \log \pi)^\top\r]$.
The operator $F^+$ in \eqref{eq: obj} denotes the Moore-Penrose pseudoinverse of $F$ (see
e.g. Appendix C.6 of \cite{cen2020fast} for the calculation of $F^+\nb_{\th_{sa}}E$). The vector
$c = (c_s)_{s\in\S}\in\R^{|\S|}$ in \eqref{eq: obj} depends on state $s$ and is independent of action $a$. When the
policy is represented by the soft-max function, the explicit form of $c$ does not influence the
update of the policy. The vanilla gradient descent algorithm for the minimization problem \eqref{eq:
  objobj} takes the form
\begin{equation}\label{eq: gd_0}
  V^{k+1}_s = V^k_s - \eta_V G_{V_s}(V^k,\pi^k), \quad \pi^{k+1}_{sa} \propto \pi^k_{sa}e^{-\eta_\pi G^{(0)}_{\th_{sa}}(V^k,\pi^k)},
\end{equation}
where $\eta_V$ and $\eta_\pi$ are the learning rates.

%We will prove in Section \ref{sec: thm} that for sufficiently large $\beta$, the fixed point of the
%above algorithm is close to the optimal policy $\pi^*$.

Although we show in Section \ref{sec: thm} that the above algorithm converges to a policy that is
close to the optimal policy $\pi^*$, the trajectory towards the minimizer is often not optimal. When
$\ell_s<0$, the path from $\pi$ towards $\pi^*$ may detour if the algorithm is directed according to
$G^{(0)}_{\th_{sa}}$. As shown in Figure \ref{fig: figh0}, the error $\pi - \pi^*$ can increase at
the initial stage of the algorithm. \tcb{Intuitively, since $\nb_\pi E = \nb_\pi \ll \ell \rl_\rho^2$, the gradient $\nb_\pi E$ of the objective function in $\pi$ tries to minimize the residual norm $\ll \ell\rl_\rho$ in the policy space. When $\ell<0$, $V$ underestimates the true value function $V^\pi_\lam$.  Therefore, in order to reduce the residual $\ll \ell \rl_\rho$ in the policy space, $\nb_\pi \ll \ell \rl_\rho^2$ will lead to a direction that reduces $V^\pi_\lam$, which is undesirable.  On the other hand, when $\ell$ is non-negative, $V$ overestimates $V^\pi_\lam$. Hence, $\nb_\pi \ll \ell \rl_\rho^2$ will lead to a direction that increases $V^\pi_\lam$ and, therefore, reduces $\ll \ell \rl_\rho$, which is the desired direction. }

Another
way to understand this aforementioned detour is through the gradient of the objective function. 
Notice that the gradient $\pt_{\pi_{sa}}E$ of the objective
function $E(V,\pi)$ in $\pi$ has the same form as $G^{(0)}_{\th_{sa}}$ in \eqref{eq: obj}. Therefore, 
\begin{equation}\label{eq: gd_1}
\begin{aligned}
    -\pt_{\pi_{sa}} E(V,\pi) =& -\b(\ell\odot\rho)^\top \pt_{\pi_{sa}}\l[(I-\g P^\pi)(V - (I-\g P^\pi)^{-1}(r^\pi-\lam\H(\pi)))\r]   \\
    =& -\b(\ell\odot\rho)^\top \pt_{\pi_{sa}}\l[(I-\g P^\pi)(V - V_\lam^\pi)\r] \\
    =& \b(\ell\odot\rho)^\top(I-\g P^\pi)\pt_{\pi_{sa}}( V_\lam^\pi) + \b\g (\ell\odot\rho)^\top \pt_{\pi_{sa}}(P^\pi)(V - V_\lam^\pi),
\end{aligned}
\end{equation}
where $\ell\odot\rho$ is the entry-wise product of $\ell$ and $\rho$, i.e.,
$(\ell\odot\rho)_s=\rho_s\ell_s$. Note first that the second term of the RHS of \eqref{eq: gd_1}
contains $(V - V_\lam^\pi)$ and $(V - V_\lam^\pi)$ is small because the objective function
\eqref{eq: objobj} pushes $V$ to the true value function $V_\lam^\pi$ for sufficiently large
$\b$. In fact, for any fixed $\pi$, the local fixed point of the $V$ updates satisfies
$G_{V_s}=0$, where $G_{V_s}$ is defined in \eqref{eq: obj}. This implies that
\[
V - V_\lam^\pi = V - (I-\g P^\pi)^{-1}(r^\pi-\lam\H(\pi)) = \frac{1}{\b}(I-\g P^\pi)^{-1}\l[\t{\rho}\odot [(I-\g
P^\pi)^{-\top}\rho]\r]\sim O\l(\frac1\b\r),
\]
where $\t{\rho}_s = 1/\rho_s$. Hence, the second term of the RHS of \eqref{eq: gd_1} is of order $O(1/\b)$. When $\b$ is large, the
gradient $-\pt_{\pi_{sa}} E(V,\pi)$ is dominated by the first term $\b(\ell\odot\rho)^\top(I-\g
P^\pi)\pt_{\pi_{sa}}( V_\lam^\pi)$, which is equivalent to $\beta\ell\rho(1-\g)\pt_\pi( V^\pi_\lam)$
in the one-dimensional case. Note that $\pt_\pi ( V_\lam^\pi)$ is the steepest ascent direction for
maximizing $V^\pi_\lam$. Therefore, when $\ell<0$, the term $\beta\ell\rho(1-\g)\pt_\pi(
V^\pi_\lam)$ is the opposite direction of the steepest ascent, which implies that the gradient
descent algorithm based on $-\pt_{\pi} E$ does not move towards maximizing $V^\pi_\lam$. This
illustrates why the algorithm \eqref{eq: gd_0} can take a detour to the optimal policy $\pi^*$.

%We will leave the rigorous proof for future research.

\begin{figure}[h!]
  \centering
  \includegraphics[width=1\linewidth]{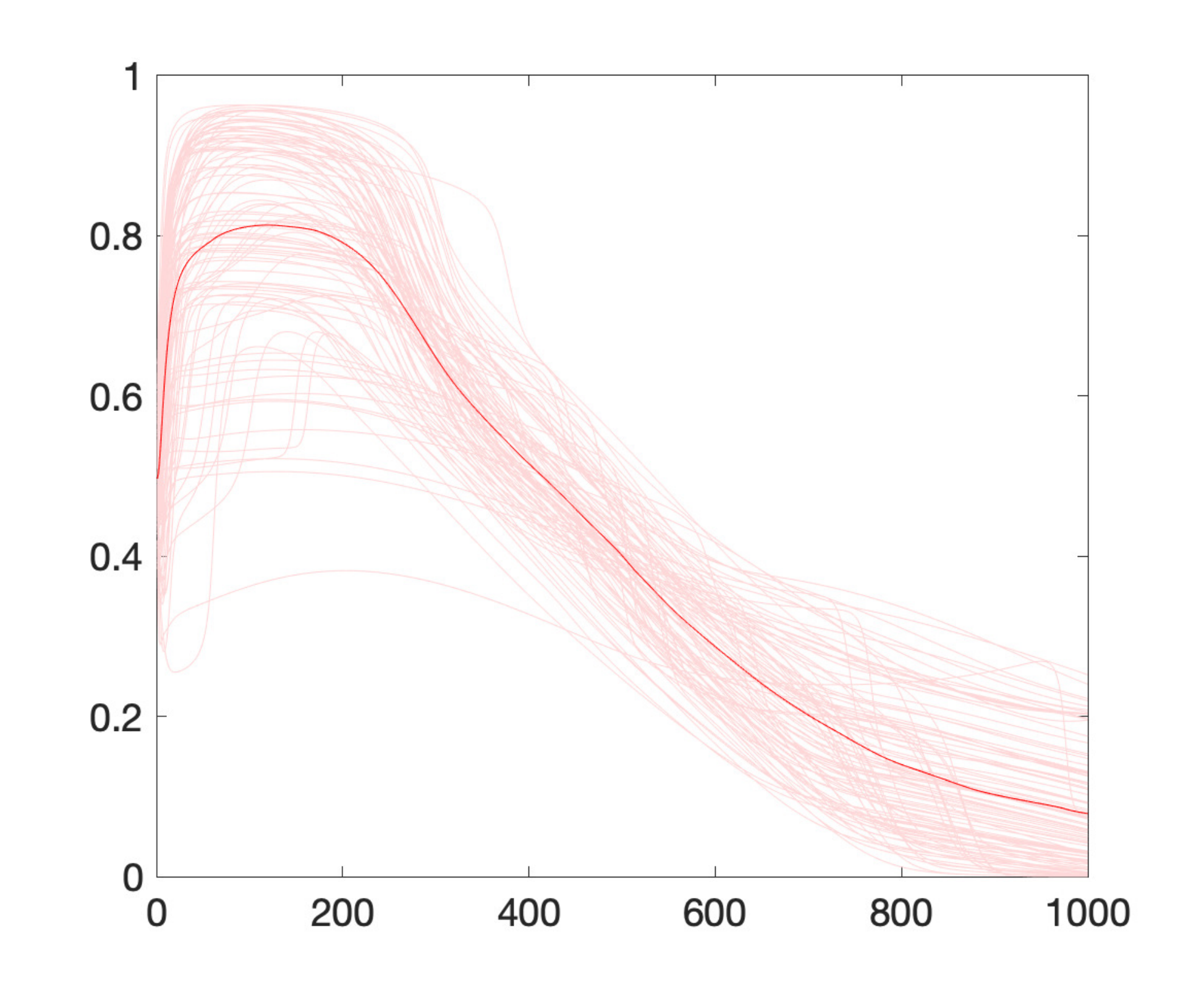}
  \caption{The left plot shows the error $\pi^k - \pi^*$ from the vanilla gradient descent method \eqref{eq: obj}, i.e., Algorithm \ref{algo: V GD}
    with $h^{(0)}(\ell_s) = \ell_s$. We test the algorithm for an MDP with $5$ states and $2$
    actions, and set $\b = 10, \lam = 0$ and the learning rate $\eta_V = \eta_\pi = 1/(4\b)$. The red
    line is the mean over the 100 simulations. \tcb{The middle plot shows the existence of negative residuals. It plots $1$ when there exist negative residuals at step $k$ and $0$ when the residuals are non-negative at all states. The blue line is the mean value over the $100$ simulations. The right plot collects the average lines of the left two plots in one figure. } One can see from the left plot that the error increases at
    the initial stage and then decreases at the latter stage. \tcb{From the middle plot, one can see that there always exists at least one negative residual among the $5$ states at the initial stage and then the residuals become all non-negative at the latter stage. On the right plot, one can see that the error increase is closely related to the existence of negative residuals. When the residuals have fewer or no negative values, the error decreases rapidly as expected. } }
  \label{fig: figh0}
\end{figure}

To address this issue, we propose two methods to improve the efficiency of the algorithm.
\begin{itemize}
\item The {\it clipping} method. The idea is to suppress $\ell_s$ when $\ell_s < 0$, i.e., the
  policy update is based on a clipping modification,
  \[
  G^{(1)}_{\th_{sa}} =  \rho_s\beta \ell_s\mathds{1}_{\ell_s>0} \l[- \g \sum_tP^a_{st}V_t - r_{sa} +\lam\log\pi_{sa}\r] .
  \] 
  This algorithm can be viewed as a gradient descent method for the optimization problem \eqref{eq:
    obj} with a metric $(\mathrm{id},\mathds{1}_{\ell_s>0}\cdot\mathrm{id})$ on $(V_s,\th_{sa})$. 
\item The {\it flipping} method. The idea is to flip the sign of $\ell_s$ when $\ell_s<0$,
  i.e., the policy update is based on 
  \[
  G^{(2)}_{\th_{sa}} = \rho_s\b |\ell_s|\l[- \g \sum_tP^a_{st}V_t - r_{sa} +\lam\log\pi_{sa}\r].
  \] 
\end{itemize}
From the analysis of \eqref{eq: gd_1}, we see that the vanilla gradient descent with
$G^{(0)}_{\th_{sa}}$ in \eqref{eq: obj} would make the policy worse locally. Intuitively, the
clipping method with $G^{(1)}_{\th_{sa}}$ stops updating the policy when $\ell_s<0$, while the
flipping method further improves the policy $\pi$ because $G^{(2)}_{\th_{sa}}$ always has the same
direction as $\nb_\pi V^\pi$.

Three different versions (vanilla, clipping, and flipping) of the variational actor-critic
based on the objective function \eqref{eq: objobj} can be summarized as follows:
\begin{equation}\label{eq: algorithm}
  V^{k+1} = V^k - \eta_VG_V(V^k,\pi^k),\quad \pi^{k+1}_{sa}\propto \pi^k_{sa}\exp(-\eta_\pi G^{(i)}_{\th_{sa}}(V^k,\pi^k) ), \quad i = 0,1,2,
\end{equation}
where $G_V$ and $G^{(i)}_\th$ represent the gradients with respect to $V$ and $\th$,
\begin{equation}\label{eq: updates}
  \begin{aligned}
    G_{V_s}(V,\pi) =& -\rho_s + \b\l(\ell_s\rho_s-\g\sum_t P_{ts}^\pi\ell_t\rho_t\r), \\
    G^{(i)}_{\th_{sa}}(V,\pi) =& \b \rho_s h^{(i)}(\ell_s)\l[-\g \sum_{t} P_{st}^a V_t - r_{sa} +\lam\log\pi_{sa}\r] + c_s, \quad i = 0,1,2,
  \end{aligned}
\end{equation}
and $h^{(i)}$ is defined as
\begin{equation}\label{def of h}
    h^{(0)} = x,\quad h^{(1)}(x) = x\mathds{1}_{x>0}, \quad h^{(2)}(x) = |x|.
\end{equation} 
Here $\ell = (\ell_s)_{s\in \S}$ is the Bellman residual defined in \eqref{def of L}. Under the
model-based setting (i.e., assuming that the transition dynamics is explicitly known), the algorithm
is outlined in Algorithm \ref{algo: V GD}.
\begin{algorithm}
  \caption{Variational actor-critic (model-based version)}
  \label{algo: V GD}
  \begin{algorithmic}[1]
    \REQUIRE $\eta_V, \eta_\pi$: learning rate; $\b$: penalty constant; \\
    \REQUIRE $i$: $i = 0$ (vanilla gradient descent) or $i = 1$ (clipping) or $i=2$ (flipping);
    \STATE Random initialization of $V_0,\pi_0$
    \WHILE{$V,\th$ do not converge}
        \STATE $\ell \gets (I-\g P^\pi)V - r^\pi + \lam \H(\pi)$;
        \STATE $V_s \gets V_s - \eta_V(-\rho_s + \b(\ell_s\rho_s-\g\sum_t P^\pi_{ts} \ell_t\rho_t))$;
        \STATE $\pi_{sa} \gets \pi_{sa}\exp\l[-\eta_\pi \b \rho_s h^{(i)}(\ell_s)(-\g \sum_{t} P_{st}^aV_t -r_{sa} + \lam \log\pi_{sa})\r]$, where $h^{(i)}$ is defined in \eqref{def of h};
        \STATE $\pi_{sa} \gets \frac{1}{\sum_{b}\pi_{sb}}\pi_{sa}$;
    \ENDWHILE
\end{algorithmic}
\end{algorithm}
%\LY{add $\rho$ into the code.}
%\LY{call these three versions three variants, instead of three algorithms.}

We would like to point out that the three variants $G^{(i)}_{\th_{sa}}$ for $i = 0,1,2$ coincide
when $\ell_s \geq 0$. Lemma \ref{lemma: around steady} demonstrates that $\ell_s$ is larger than $0$
when $(V,\pi)$ achieves the fixed point of the algorithm. Therefore, the three variants $G^{(i)}_{\th_{sa}}$ with
$i = 0,1,2$ are different only at the initial stage of the optimization process and become the same
at the latter stage with $\ell_s>0$. The vanilla gradient descent might go to a worse policy first
and then go to the direction that maximizes $V^\pi_\lam$; the clipping method might stop updating
the policy until $V$ near its local fixed point with $\ell_s>0$; the flipping method would go all
the way along the direction maximizing $V_\lam^\pi$. Although the three variants converge to the fixed point with different dynamics, they eventually converge to the same fixed point.

\subsection{Model-free setting} \label{sec: stochastic algo}
%Algorithm \ref{algo: V GD} is applicable in the model-based setting.

When the transition dynamics $P^\pi$ is unknown as in the model-free RL, one only has access to one
(or multiple) off-policy trajectory $\{(s_t,a_t,r_t)\}_{t=1}^T$ generated by a behavior policy
$\pi_b(s,a)$. Algorithm \ref{algo: V GD} can in principle be generalized to the model-free setting
if one updates $(V,\pi)$ based on an unbiased estimate of the gradient \eqref{eq: updates} (see
Appendix \ref{appendix: SGD_V} for the stochastic algorithm in $V$-formulation). However, reweighting is 
necessary in order to correct the difference between the behavior policy $\pi_b$ and the target policy
$\pi$ when approximating the term $P^\pi V$. The reweighting method, although unbiased, would cause
instability in the process of SGD (\cite{an2020resampling, schlegel2019importance}).

%In what follows, we assume that the approximate value function $V(s,\o)$ is parametrized by $\o$ and
%the policy $\pi(s,a,\th)$ is parametrized by $\th$.

It is instead preferred to use the $Q$-formulation as there is no need to correct the behavior
policy. The stochastic algorithm in the $Q$-formulation is based on the following objective
function:
\begin{equation}\label{eq: q}
\underset{(s,a)\sim\rho}{\E}\l[ -Q(s,a) + \frac{\b}2\l(Q(s,a) - \underset{s'\sim
    P^a(\cdot|s,a)}{\E}\l[\g\sum_a ( Q(s',a)- \lam \log\pi(s',a))\pi(s',a) |s,a\r] - r(s,a)\r)^2\r],
\end{equation}
where $\rho>0$ is the positive stationary distribution from behavior policy $\pi_b$. By using the
$Q$-formulation, one can directly use the trajectory $\{(s_t,a_t,r_t)\}_{t=1}^T$ without reweighting.

When $Q(s,a,\o)$ is parametrized by $\o$ and $\pi(s,a,\th)$ is parametrized by $\th$, the updates of
$(Q, \pi)$ are according to the following unbiased estimates of the gradients:
\begin{equation}\label{eq: gd_2}
    \begin{aligned}
    &(G_Q)_t = -\nb_\o Q^k(s_t,a_t) + \b L_t\l(\nb_\o Q^k(s_t,a_t) - \g\sum_a \nb_\o Q^k(s'_{t+1},a)\pi^k(s'_{t+1},a)\r),\\
    &(G^{(i)}_\pi)_t = \b \h{h}^{(i)}(L_t)\l(-\g\sum_a(Q^k(s'_{t+1},a) - \lam\log\pi^k(s'_{t+1},a)-\lam)\nb_\th\pi^k(s'_{t+1},a) \r),
\end{aligned}
\end{equation}
where $Q^k(s,a) = Q(s,a,\o_k)$, $\pi^k(s,a) = \pi(s,a,\th_k)$ and $L_t$ is the unbiased Bellman
residual,
\[
L_t = Q^k(s_t,a_t) - r_t - \g\sum_a \l(
Q^k(s_{t+1},a)-\lam\log\pi^k(s_{t+1},a)\r)\pi^k(s_{t+1},a).
\]
Here the next state $s'_{t+1}$ in $G_Q$ and $G_\pi$ needs to be uncorrelated with the next state
$s_{t+1}$ in the trajectory. Since it is usually unrealistic to generate another independent sample at
state $s_t$ with action $a_t$, the BFF algorithm is proposed in \cite{zhu2020borrowing} to generate
an approximate $s'_{t+1}$
\[
s'_{t+1} = s_t + (s_{t+2} - s_{t+1}).
\]
It is shown in \cite{zhu2020borrowing} that when the underlying dynamics changes smoothly with respect to the
actions and states, the BFF approximation is close to the independent sample in expectation.  Furthermore,
$\h{h}^{(i)}(x)$ is defined as follows,
\begin{equation}\label{def of hp}
  \h{h}^{(1)}(L_t) = L_t\mathds{1}_{\h{\ell}_{s_t}>0}, \quad \h{h}^{(2)}(L_t) =\l\{ \begin{aligned}
    &L_t, \quad \h{\ell}_{s_t}>0\\
    &-L_t, \quad \h{\ell}_{s_t}<0\\
  \end{aligned}\r., \quad \text{where }\h{\ell}_s = \frac{1}{|\{s_t = s\}|}\sum_{s_t = s}L_t.
\end{equation}
Note that one cannot directly apply the
clipping or flipping function $h^{(i)}$ defined in \eqref{def of h} on the stochastic Bellman
residual $L_t$ because $\E[h^{(i)}(L_t)|s_t =s] \neq h^{(i)}( \ell_s )$. Instead, $h^{(i)}( \ell_s )$ is estimated in two steps: first, one approximates the
Bellman residual $\ell_s$ with $\h\ell_s$, and then $L_t$ is suppressed or flipped according to the
value of $\h\ell_s$. The stochastic algorithm for the $Q$-formulation is summarized in Algorithm
\ref{algo: Q}.

\begin{remark}\label{rmk: ALGAE}
A similar objective function has been used in \cite{nachum2019algaedice}. 
\tcb{Note that if one multiplies a negative constant to equation (14) of \cite{nachum2019algaedice}, then the maximum operators become minimum operators. Extend the operator $\mathcal{B}_\pi \nu(s,a) = r(s,a) + \g\mathbb{E}_{s'\sim P^a(\cdot |s,a)}[\sum_a \nu(s',a) \pi(s',a) |s,a]$ and view $\nu(s,a)$ as $Q(s,a)$, one finds that \eqref{eq: q} is equivalent to equation (14) in \cite{nachum2019algaedice} up to a constant by setting $f_*(x) = x^2/2$. In other words, our formulation \eqref{eq: q} is equivalent to the main formulation (8) in [16] when $f = x^2/2$ and $\alpha < 0$. }
The paper \cite{nachum2019algaedice} also pointed out
that the off-policy trajectory can be directly used for the policy gradient. Although
\cite{nachum2019algaedice} uses a similar trick as the clipping method for numerical experiments, it
is however only mentioned in the Appendix.  There are other two differences between the current
paper and \cite{nachum2019algaedice}. First, we propose another more efficient algorithm, flipping,
to accelerate the convergence rate. One can see the comparison of the two methods in Section
\ref{sec: numerics}. Second, we use $\h{h}^{(1)}(L_t)$ defined in \eqref{def of hp}, while
\cite{nachum2019algaedice} directly applied $h^{(1)}(x)$ defined in \eqref{def of h} to $L_t$. We
note that $\h{h}^{(1)}(L_t)$ is a better estimates to $h^{(1)}(\ell_t)$ than $h^{(1)}(L_t)$ as
explained after \eqref{def of hp}.
\end{remark}
%\cite{nachum2019algaedice} discussed that the clipping trick is critical for the algorithm to converge efficiently. 

\begin{algorithm}
  \caption{Variational actor-critic (model-free version)}
  \label{algo: Q}
  \begin{algorithmic}[1]
    \REQUIRE $\eta_V, \eta_\pi$: learning rates; $\b$: prefactor; $M$: batch size; \\  $Q(s,a,\o),\pi(s,a,\th)$: parametrized approximation of $Q(s,a),\pi(s,a)$; \\
    $\{s_t,a_t,r_t\}_{t=0}^T$: trajectory generated from off-policy $\pi_b$;
    \STATE Random initialization of $\th_0,\o_0$, $k=0$
    \WHILE{$\o,\th$ do not converge}
        \STATE $j\gets 1$, $k\gets k+1$
        \FOR{$t=(k-1)M+1,\cdots,kM$}
        \STATE $s_j = s_t$
        \STATE $L_j = Q(s_t,a_t,\o) - r_t - \g (V(s_{t+1}) - \lam\H(s_{t+1}) )$
        \STATE $s'_{t+1} \gets s_t + (s_{t+2} - s_{t+1})$
        \STATE $G_Q^j = -\nb_\o Q(s_t,a_t,\o) + \b L_j(\nb_\o Q(s_t,a_t,\o) - \g\sum_a \nb_\o Q(s'_{t+1},a_t,\o)\pi(s'_{t+1},a,\th))$
        \STATE $G_\pi^j = \b \l(-\g\sum_a(Q(s'_{t+1},a_t,\o) - \lam\log\pi(s'_{t+1},a,\th)-\lam)\nb_\th\pi(s'_{t+1},a,\th) \r)$
        \STATE $j \gets j+1$
        \ENDFOR
        \STATE $G_Q \gets \frac1M\sum_{j=1}^MG_Q^j $; $\o \gets \o - \eta_Q G_Q$
        \STATE $\h\ell_s \gets \sum_{s_j = s} L_j$
        \STATE $G^{(i)}_\pi\gets \frac1M\sum_{j=1}^M\h{h}^{(i)}(L_j)G_\pi^j $; $\th \gets \th - \eta_\pi G^{(i)}_\pi$, where $\h{h}^{(i)}$ is defined in \eqref{def of hp}
        \STATE $V(s) \gets \sum_a Q(s,a,\o)\pi(s,a,\th)$; $\H(s) \gets \sum_a \pi(s,a,\th)\log\pi(s,a,\th)$
    \ENDWHILE
\end{algorithmic}
\end{algorithm}

Specifically, if $\pi$ is the soft-max function of $\th$, then the $\pi$ updates based on
$(G^{(i)}_\pi)_t$ in \eqref{eq: gd_2} can be simplified to,
\[
\begin{aligned}
  &(f^i_{sa})_t = \g\pi^k(s,a)\b \h{h}^{(i)}(L_t) \l[V^k(s) - Q^k(s,a) + \lam\log\pi^k(s,a) - \lam
    \H(\pi^k_s)\r]\mathds{1}_{s = s_{t+1}'},\\
  &\pi^{k+1}(s,a) \propto \pi^k(s,a)\exp\l(-\frac{\eta_\pi}{M} \sum_{t = (k-1)M+1}^{kM}
  (f^i_{sa})_t\r).
\end{aligned}
\]

%==============================================================================================================
\section{Fixed Point Estimates}\label{sec: thm}

%\LY{define fixed point.}
%\LY{some place one needs to refer to formulation equation rather than algorithm}

We define $(\Vinf, \piinf)$ as the {\em fixed point} of Algorithm \ref{algo: V GD} if
\begin{equation}\label{eq: def of fixedpt}
  \Vinf = \Vinf - \eta_V G_V(\Vinf, \piinf), \quad \piinf_{sa} \propto \piinf_{sa}\exp\l(-\eta_\pi
  G^{(i)}_{\th_{sa}}(\Vinf, \piinf)\r),
\end{equation}
where $G_V$ and $G^{(i)}_\th$ are defined in \eqref{eq: updates}. Specifically, for the
non-regularized MDP, i.e., $\lam = 0$, the fixed point $\piinf$ is exactly the optimal policy
$\pi^*$ when $\b$ is sufficiently large; for the regularized MDP, i.e., $\lam > 0$, the fixed point
$\piinf$ is close to $\pi^*$ for large $\b$ and small $\lam$. %In this section, we show that for sufficiently large $\b$, the fixed point \eqref{eq: def of fixedpt} is close to the optimal policy $\pi^*$.

We analyze the non-regularized MDP and regularized MDP in Section \ref{sec: non_reg} and Section
\ref{sec: reg}, respectively.  For $\lam = 0$, we prove in Lemma \ref{lemma: gthneq0} that when $\b$ is sufficiently large and $V$
achieves its fixed point $\Vinf$, the gradient of the policy $G^{(i)}_\th(\Vinf, \pi)$ cannot be equal to
$0$ for any $ \pi$. This implies that the fixed point $\piinf$ of the
policy updates is on the boundary of the probability simplex, i.e., $\piinf$ is a deterministic
policy. Since all deterministic policies form a discrete set and the optimal policy $\pi^*$ for a
non-regularized MDP is also a deterministic policy, there exists $\b_0>0$ such that for all
$\b>\b_0$, the fixed point $\piinf$ of the algorithm is the optimal policy $\pi^*$. On the other
hand, for $\lam>0$, the fixed point $\piinf$ is a stochastic policy. Therefore, one can only prove
that for $\b > \b_0$ and $0<\lam<\lam_0$, the fixed point is close to the non-regularized optimal
policy $\pi^*$.

Before analyzing the fixed point of the algorithm, we state some basic properties of the matrix $(I
- \g P^\pi)$ in Proposition \ref{prop: transmatrix}. In Lemma \ref{lemma: around steady}, we prove
that the Bellman residual $\ell(V,\pi)$ defined in \eqref{def of L} is always positive at the fixed
point $(\Vinf, \piinf)$. This implies that $G^{(i)}_{\th_{sa}}$ takes the same form at the fixed
point $(\Vinf,\piinf)$. Hereafter, we shall omit the index $i$ of $G^{(i)}_{\th_{sa}}$ for
notational simplicity.

\begin{proposition}\label{prop: transmatrix}
For any transition matrix $P$ and positive vector $c$, the following inequalities hold,
\begin{equation}\label{eq: ineq1}
    \mathbf{0}< (I-\g P)^{-1}c < \frac{\max_i
  c_i}{1 - \g} \mathbf{1}, \quad\mathbf{0}< (I-\g P)^{-\top}c < \frac{\sum c_s}{(1-\g)}\mathbf{1}.
\end{equation} 
For any constant $c$, 
\begin{equation}\label{eq: ineq2}
    \text{if }(I-\g P )x \leq c\mathbf{1}, \quad \text{ then  }x\leq \frac{c}{1-\g}\mathbf{1}.
\end{equation}

\end{proposition}
See Appendix \ref{appendix: proof of prop} for the proof. 
%\LY{move proof to appendix}

\begin{lemma}\label{lemma: around steady}
The fixed point $(\Vinf, \piinf)$ of Algorithm \ref{algo: V GD} satisfies $\ell(\Vinf, \piinf) > 0$
with $\ell(V,\pi)$ defined in \eqref{def of L}.
\end{lemma}
\begin{proof}
  The $V$ update achieves its fixed point when $G_V(\Vinf,\piinf) = 0$, which gives,
  \[
  -\rho + \b(I-\g P^{\piinf})^\top\l[\l((I-\g P^{\piinf}) \Vinf - r^\pi + \lam \H(\piinf)\r)\odot\rho \r] = \mathbf{0}.
  \]
  This leads to,
  \[
  \ell(\Vinf,\piinf) = \l[(I-\g P^{\piinf}) \Vinf - r^\pi + \lam \H(\piinf) \r] = \frac1\b\t{\rho}\odot\l[(I-\g P^{\piinf})^{-\top}\rho\r],
  \]
where $\t{\rho}_s = 1/\rho_s>0$.   By \eqref{eq: ineq1} of  Proposition \ref{prop: transmatrix}, all elements of $\ell$ are positive, which completes the proof. 
\end{proof}
Since the three variants $G^{(i)}_{\th_{sa}}$ with $i = 0,1,2$ defined in \eqref{eq: updates} are
the same when $\ell_s>0$ and $\ell_s$ is positive at the fixed point $(\Vinf, \piinf)$, they share the same fixed
point.
%In addition,  the evolution of the variants are the same around the fixed point.

%============================
\subsection{Fixed point for the non-regularized MDP} \label{sec: non_reg}
Recall that the non-regularized MDP refers to the case where $\lam =0$. Below we prove that there
exists a threshold $\b_0$, such that for all $\b>\b_0$, the fixed point $\piinf$ of the policy
updates is the optimal policy $\pi^* =\argmax_\pi V^\pi = \argmax_\pi (I-\g P^\pi)^{-1} r$. For simplicity, we assume that the distribution in \eqref{eq: objobj} is the uniform distribution, i.e., $\rho_s = 1/|\S|$ in this section. The results can be extended to general distribution $\rho$ (see Remark \ref{rmk: general rho} for details). \tcb{Besides, we always assume that the action gap is strictly positive, i.e., let $a^*_s = \max_a (r_{sa} + \g \sum_tP^a_{st}V^*_t)$, then 
\[
\max_{a\neq a^*_s} (r_{sa} + \g \sum_tP^a_{st}V^*_t) < r_{sa^*_s} + \g \sum_tP^{a^*_s}_{st}V^*_t, \quad \forall s\in\S.
\]}

The fixed point $(\Vinf, \piinf)$ of the algorithm is stated in Lemma \ref{lemma:
  pitoboundary}. 
  Note that  $(\Vinf, \piinf)$ satisfies similar coupled equations as the optimal solution $(V^*,\pi^*)$ in Lemma \ref{lemma:
  optimal policy}. The only difference is that
$(\Vinf, \piinf)$ satisfies $G_V(\Vinf, \piinf) = \mathbf{0}$ while $(V^*,\pi^*)$ satisfies the Bellman equation $V^* =
(I-\g P^{\pi^*})^{-1}r$.  Note that $G_V = \mathbf{0}$ can be written as
\[
V(\pi) = (I-\g P^\pi)^{-1}r + \frac1\b (I-\g P^\pi)^{-1}(I - \g P^\pi)^{-\top}\mathbf{1}.
\]
When $\b$ is sufficiently large, $\Vinf(\pi)$ approaches the true value function $V^\pi = (I-\g
P^\pi)^{-1}r$. On the other hand, we prove in Lemma \ref{lemma: as_asstar} that there exists a
threshold $\alpha$, such that $|V^*_s - \Vinf_s | \leq \alpha$ for $\forall s\in\S$, then $\piinf =
\pi^*$. Combining the above lemmas, one concludes in Theorem \ref{thm: non-reg} that
the fixed point $\piinf$ is the optimal policy $\pi^*$ as long as $\b >
\frac{|\S|}{(1-\g)^2\alpha}$, where $\alpha$ is a constant related to the optimal solution $(V^*,\pi^*)$. Note that the
lower bound for the prefactor $\b$ is not sharp, and we shall see in Section \ref{sec: numerics} that the
algorithm converges numerically to the optimal policy $\pi^*$ with much smaller $\b$.

\begin{lemma}\label{lemma: pitoboundary}
The fixed point $(\Vinf, \piinf)$ of Algorithm \ref{algo: V GD} satisfies the following coupled equations,
\begin{equation}\label{eq: pitobdy}
    \l\{
\begin{aligned}
    &G_V(\Vinf,\piinf) = 0,\\
    &\piinf_{sa} = \l\{\begin{aligned} &1,\quad a= a_s ; \\ &0,\quad a\neq a_s,\end{aligned}\r.\quad
\text{where} \ a_s = \argmax_a \l(\g\sum_t P^{a}_{st} \Vinf_t + r_{sa}\r).
\end{aligned}
\r.
\end{equation}
\end{lemma}
\begin{proof}
Since $V$ is updated as follows
\[
V_{k+1} = V_k - \eta_VG_V(V_k,\pi_k),
\] 
the only fixed point for the above update satisfies $G_V(\Vinf,\piinf) = 0$.  Hence, it is
equivalent to prove that if $V_k \equiv \Vinf$ in the $\pi$ updates,
\[
(\pi_{k+1})_{sa} \propto (\pi_k)_{sa} \exp\l(-\eta_\pi G_{\th_{sa}}(V_k,\pi_k)\r),
\]
then $\lim_{k\to\infty}\pi_k = \piinf$ with $\piinf$ stated in the lemma.

We prove in Lemma \ref{lemma: around steady} that the Bellman residual $\ell_s>0$ always holds
at the fixed point $(\Vinf, \piinf)$, so $\pi_k$ is updated as follows around the fixed point,
\begin{equation*}
  \begin{aligned}
    &(\pi_{k+1})_{sa} \propto (\pi_k)_{sa} \exp\l(\eta_\pi\l(\g \sum_t P^a_{st}(V_k)_t +r_{sa}\r) \r).
  \end{aligned}
\end{equation*}
Plugging $V_k \equiv \Vinf$ into the above equation gives
\[
(\pi_{k+1})_{sa} \propto (\pi_k)_{sa} \exp\l[\eta_\pi \l( \g\sum_t P^a_{st}\Vinf_t + r_{sa} - \l(\g\sum_t P^{a_s}_{st} \Vinf_t + r_{s \, a_s}\r)\r)\r],
\]
where $a_s$ is defined in \eqref{eq: pitobdy}.
Then one has
\[
\l\{
\begin{aligned}
    &(\pi_{k+1})_{sa} \propto (\pi_k)_{sa}, \quad\text{for }a=a_s;\\
    &(\pi_{k+1})_{sa} \propto (\pi_k)_{sa}\exp(f_{sa}(\Vinf)),\quad \text{for }a\neq a_s,
\end{aligned}
\r.
\]
where $f_{sa}(\Vinf) < 0$. Hence, the $\pi$ updates can be equivalently written as
\[
\l\{
\begin{aligned}
    &(\pi_{k})_{sa} \propto (\pi_0)_{sa}, \quad\text{for }a=a_s;\\
    &(\pi_k)_{sa} \propto (\pi_0)_{sa}\exp(k f_{sa}(\Vinf)),\quad \text{for }a\neq a_s.
\end{aligned}
\r.
\]
Notice that as $f_{sa}(\Vinf) < 0$,  $\lim_{k\to\infty}(\pi_0)_{sa}\exp(k f_{sa}(\Vinf)) = 0$. Therefore, 
\[
\lim_{k\to\infty}(\pi_{k})_{sa} = \l\{
\begin{aligned}
   &1, \quad a =a_s;\\
   &0, \quad a\neq a_s,
\end{aligned}
   \r. 
\]
which completes the proof. 
\end{proof}

\begin{lemma}\label{lemma: optimal policy}
  The maximum and maximizer $(V^*,\pi^*)$ of the optimization problem \eqref{eq:pg}  satisfy the following coupled equations:
\begin{equation}\label{eq: optimal policy}
      \l\{
\begin{aligned}
    &V^* = r^{\pi^*} + \g P^{\pi^*}V^*,\\
    &\pi^*_{sa} = \l\{\begin{aligned} &1,\quad a= a^*_s ; \\ &0,\quad a\neq a^*_s,\end{aligned}\r. \quad
\text{where} \ a_s^* = \argmax_{a} \l(r_{sa} + \g \sum_t P^a_{st}V^*_t\r).
\end{aligned}
\r.
\end{equation}
\end{lemma}

\begin{proof}
The maximum $V^*$ also satisfies the optimal Bellman equation as follows \cite{sutton2018reinforcement}, 
\begin{equation}\label{eq: pistar}
    V^*_s = \max_{a} \l(r_{sa} + \g  \sum_{t}P^a_{st} V^*_t\r), \quad \text{for }\forall s\in\S.
\end{equation}
For $a_s^*$ and $\pi^*$ defined in \eqref{eq: optimal policy}, the following equality holds
\[
r^{\pi^*}_s+ \g \sum_t P^{\pi^*}_{st}V^*_t =  \sum_{a}\pi^*_{sa}\l(r_{sa} + \g\sum_t P^a_{st} V^*_t\r) = \max_a \l(r_{sa} + \g\sum_t P^a_{st} V^*_t\r) = V^*_s .
\]
Hence, $V^*$ satisfies the Bellman equation $V^* = r^{\pi^*} + \g P^{\pi^*}V^*$, which completes the proof.
\end{proof}

\begin{lemma}\label{lemma: as_asstar}
For any value functions $\Vinf, V^*\in\R^{|\S|}$, let $a_s = \argmax_a \l(r_{sa} + \g\sum_t P^a_{st}\Vinf_t\r)$ and  
\\$a_s^*=\argmax_a\l(r_{sa} + \g \sum_t P^a_{st}V^*_t\r)$ be the maximizers, then there exists
\[
\e'=\min_s\l(r_{sa_s^*} + \g\sum_{t}P^{a^*_s}_{st}V^*_t-\max_{a\neq a^*_s}\l(r_{sa} + \g\sum_{t}P^{a}_{st}V^*_t\r)\r),
\]
such that as long as $|\Vinf_s - V^*_s| < \e = \e'/3$ for $\forall s$, then $a_s = a_s^*$ for $\forall s$.
\end{lemma}
The above lemma tells us that when the fixed point $\Vinf$ is close to $V^*$, then $a_s$ and $a_s^*$
defined in Lemmas \ref{lemma: optimal policy} and \ref{lemma: pitoboundary} are the same.

\begin{proof}
If $|\Vinf_t - V^*_t| \leq \e $, then
\[
    V^*_t-\e \leq \Vinf_t, \quad -(V^*+\e) \leq -\Vinf_t, \quad \text{for }\forall t\in\S, 
\]
which further leads to,
\[
r_{sa_s^*} + \g\sum_{t}P^{a^*_s}_{st}(V^*_t-\e ) \leq r_{sa_s^*} + \g\sum_{t}P^{a^*_s}_{st}\Vinf_t,
\]
\[
-r_{sa'}-\g\sum_t P^{a'}_{st}(V^*+\e) \leq  -r_{sa'}-\g\sum_t P^{a'}_{st}\Vinf_t, \quad \text{for }a'\notin a^*_s.
\]
Summing the two inequality together gives,
\[
r_{sa_s^*} + \g \sum_{t}P^{a^*_s}_{st}V^*_t - \l(r_{sa'} + \g \sum_t P^{a'}_{st}V^*\r) - 2\g\e \leq r_{sa_s^*} + \g \sum_{t}P^{a^*_s}_{st}\Vinf_t- \l(r_{sa'}+\g \sum_t P^{a'}_{st}\Vinf_t\r).
\]
Since the LHS $\geq \e' - 2\g\eps = \e' - \frac{2\g}{3}\e' > 0$, one has
\[
r_{sa_s^*} + \g \sum_{t}P^{a^*_s}_{st}\Vinf_t- \l(r_{sa'}+\g \sum_t P^{a'}_{st}\Vinf_t\r) > 0, \quad \text{for }\forall a'\neq a^*_s.
\]
The above inequality implies that $a^*_s = \argmax_a \l(r_{sa} + \g\sum_t P^a_{st}\Vinf_t\r) = a_s$, which completes the proof.

\end{proof}

\begin{theorem}\label{thm: non-reg}
Let $(V^*,\pi^*)$ be the maximum and maximizer of \eqref{eq:pg}, and $\alpha$ is a positive constant s.t.
$\alpha<\frac13g(V^*, \pi^*)$. There exists a
constant $\b_0 = \frac{|\S|}{(1-\g)^2\alpha}$, such that $\forall \b>\b_0$, the fixed 
point $(\Vinf, \piinf)$ of Algorithm \ref{algo: V GD} is close to $(V^*,\pi^*)$ in the sense
that
\[
|\Vinf_s - V^*_s| \leq \alpha, \text{ for }\forall s\in\S,\quad \piinf = \pi^*,
\]
where $$g(V^*, \pi^*) = \min_{s\in\S}\l(r_{sa_s^*} + \g\sum_{t}P_{st}^{a^*_s}V^*_t-\max_{a\neq a^*_s}(r_{sa} + \g\sum_tP_{st}^{a}V^*_t)\r)$$ with $a_s^* = \argmax\l(r_{sa} + \g\sum_tP_{st}^{a}V^*_t\r)$.
\end{theorem}
\begin{proof}
The fixed point $(\Vinf, \piinf)$ satisfies $G_V(\Vinf, \piinf) = 0$, which gives
\begin{equation}\label{eqn: piinf}
\begin{aligned}
(I-\g P^{\piinf})\Vinf = r^{\piinf} + \frac1\b(I - \g P^{\piinf})^{-\top}\mathbf{1}. 
\end{aligned}
\end{equation}
Subtracting the value function in \eqref{eq: optimal policy} $(I - \g P^{\pi^*})V^* = r^{\pi^*}$ from the
one in \eqref{eqn: piinf} yields,
\begin{equation}\label{eq: pf_2}
    (I - \g P^{\piinf})\Vinf - r^{\piinf}  - (I - \g P^{\pi^*})V^* + r^{\pi^*} = \frac1\b(I - \g P^{\piinf})^{-\top}\mathbf{1} .
\end{equation}
By the definition of $\piinf$ and $\pi^*$ in Lemmas \ref{lemma: pitoboundary} and \ref{lemma: optimal policy}, one has 
\[
r^{\piinf} + \g P^{\piinf}\Vinf \geq r^{\pi^*} + \g P^{\pi^*}\Vinf, \quad r^{\pi^*} + \g P^{\pi^*}V^* \geq r^{\piinf} + \g P^{\piinf}V^*.
\]
Applying the above two inequalities to \eqref{eq: pf_2} yields
\[
(I-\g P^{\piinf})(\Vinf - V^*) \leq \frac1\b(I - \g P^{\piinf})^{-\top}\mathbf{1} \leq (I-\g P^{\pi^*})(\Vinf - V^*).
\]
By \eqref{eq: ineq1} of Proposition \ref{prop: transmatrix}, one has $\mathbf{0}<\frac1\b(I - \g P^{\piinf})^{-\top}\mathbf{1}
< \frac{|\S|}{\b(1-\g)}\mathbf{1}$. Therefore,
\[
(I-\g P^{\piinf})(\Vinf - V^*) \leq \frac{|\S|}{\b(1-\g)}\mathbf{1}, \quad (I-\g P^{\pi^*})(\Vinf - V^*) > \mathbf{0}.
\]
Applying \eqref{eq: ineq2} of Proposition \ref{prop: transmatrix} to the above two inequalities gives $\bf{0} \leq \Vinf - V^* \leq
\frac{|\S|}{\b(1-\g)^2}\mathbf{1}$. By Lemma \ref{lemma: as_asstar},
when $\frac{|\S|}{\b(1-\g)^2} \leq \a = \frac{1}{3}g(V^*, \pi^*)$, then $a_s = a^*_s$, which implies
$\piinf=\pi^*$.
\end{proof}

\begin{remark}\label{rmk: general rho}
For general $\rho$, one has 
\[
(I-\g P^{\piinf})\Vinf = r^{\piinf} + \frac{1}{\b} \t{\rho} \odot \l[(I-\g P^{\piinf})^{-\top}\rho\r] . 
\]
where $(\t{\rho})_s = 1/\rho_s$. The proof for Theorem \ref{thm: non-reg} can be easily extended to general $\rho$. The main difference is the bound for the second term of the RHS of the above equation.
By applying Propsition \ref{prop: transmatrix}, one can bound
\[ 
{\bf 0}\leq \frac{1}{\b}(I-\g P^\pi)^{-1}\l[ \t{\rho} \odot \l[(I-\g P^\pi)^{-\top}\rho\r] \r] \leq \frac{1/\min_s \rho_s}{\b(1-\g)^2}{\bf 1},
\]
for $\forall \pi$. Hence, Theorem \ref{thm: non-reg} still holds for general $\rho$ with $\b_0 = \frac{1/(\min_s \rho_s)}{\alpha(1-\g)^2}$.  
\end{remark}
%============================
\subsection{Fixed point for the regularized MDP}\label{sec: reg}

Recall that the regularized MDP refers to the case where $\lam >0$. The regularized optimal policy
can be written in the following two equivalent forms
\[
(\pi_\lam^*)_s = \argmax_{\pi_s\in\Delta(\A)} \l(r^\pi_s +\g \sum_tP^\pi_{st}(V^\pi_\lam)_t -
\lam\H(\pi_s)\r),
\quad (\pi_\lam^*)_{sa} \propto\exp\l(\frac1\lam\l(r_{sa} + \g\sum_tP^a_{st}(V^*_\lam)_t \r)\r).
\]
In this section, we
prove that the fixed point $\Vinf$ converges to the regularized optimal value function
$\lim_{\b\to\infty}\Vinf = V_\lam^*$ as the prefactor $\b$ converges to infinity. On the other hand,
for sufficiently large prefactor $\b>\b_0$, the fixed point $\piinf$ will be close to the
non-regularized optimal policy $\pi^*$ if the entropy constant $\lam$ is small. However, when $\lam$ is relatively large, the fixed point $\piinf$ will be close to
the regularized optimal policy $\pi^*_\lam$. For simplicity, we assume the distribution in \eqref{eq: objobj} is the uniform distribution in this section. The results can be extended to general distribution $\rho$. 

In order to prove Theorem \ref{thm: reg_2}, we first prove Lemmas \ref{lemma: ineq} and \ref{lemma: reg Vinf}. The first Lemma is about the KL-divergence of two soft-max functions, and the second one gives a lower bound and an upper bound for the difference between the local fixed point of the $V$ updates and the regularized optimal value function $V^*_\lam$. Both lemmas  will be useful in the proof of Theorem \ref{thm: reg_2}. In this section, we always assume
that the learning rate $\eta_\pi > 0$ for the policy updates is sufficiently small,
so that $\eta_\pi\b \ell_s \lam$ is always less than $1$ and larger than $0$.

\begin{lemma}\label{lemma: ineq}
If $\pi =\text{soft-max}(\th)$ and  $\mu = \text{soft-max}(\o)$ with $\th, \o\in\R^d$, then the KL divergence between the probability distribution $\pi$ and $\mu$ is $D_{\text{KL}}(\pi|\o) \leq 2\max_a |\th_a -\o_a|$.
\end{lemma}
See Appendix \ref{appendix: ineq} for the proof.

\begin{lemma}\label{lemma: reg Vinf}
For any $\pi$, the solution $V$ to $G_V(V,\pi) = 0$ with $\lam >0$ satisfies the following inequalities:
\[
    V - V^*_\lam < \frac{|\S|}{\b(1-\g)^2}\mathbf{1},
\]
and 
\begin{equation}\label{eq: V lower bd}
    [(I-\g P^{\pi_\lam^*})(V - V^*_\lam)]_s    \geq \sum_a(\pi_{sa} - (\pi^*_\lam)_{sa})(\g \sum_t P^a_{st}V_t + r_{sa} - \lam \log\pi_{sa}),
\end{equation}
where $G_V$ is defined in \eqref{eq: updates}, and $(V^*_\lam, \pi^*_\lam)$ are the maximum and
maximizer to \eqref{eq: goal}.
\end{lemma}

\begin{proof}
Since $(V^*_\lam, \pi^*_\lam)$ satisfies the regularized Bellman equation $(I-\g
P^{\pi^*_\lam})V^*+\lam\H(\pi^*)=r^{\pi^*_\lam}$, subtracting it from $G_V(V, \pi) = 0$ gives
\begin{equation}\label{eq: pf_6}
  (I-\g P^{\pi})V +\lam\H(\pi)  - (I-\g P^{\pi^*_\lam})V_\lam^* - \lam\H(\pi_\lam^*) = r^{\pi} - r^{\pi^*_\lam} + \frac1\b(I-\g P^{\pi})^{-\top}\mathbf{1}.
\end{equation}
Note that the regularized optimal policy $\pi^*_\lam$ can also be represented by 
\[
(\pi_\lam^*)_s = \argmax_{\pi_s\in\Delta(\A)} (r_s^\pi + \g \sum_tP^\pi_{st}(V^*_\lam)_t - \lam\H(\pi_s)).
\]
Therefore, $r^{\pi^*_\lam} + \g P^{\pi_\lam^*}V^*_\lam - \lam\H(\pi_\lam^*) \geq r^\pi + \g P^{\pi}V^*_\lam -
\lam\H(\pi)$. Plugging it to \eqref{eq: pf_6} leads to
\[
(I-\g P^{\pi})(V - V_\lam^*) \leq  \frac1\b(I-\g P^{\pi})^{-\top}\mathbf{1}.
\]
Further, By \eqref{eq: ineq1} in Proposition \ref{prop: transmatrix}, one has $\mathbf{0}< \frac{1}{\b}(I-\g
P^\pi)^{-\top}\mathbf{1} <\frac{|\S|}{\b(1-\g)}\mathbf{1}$ for $\forall \pi$.  Therefore,
\[
(I-\g P^{\pi})(\Vinf - V^*_\lam) <\frac{|\S|}{\b(1-\g)}\mathbf{1}.
\]
Applying \eqref{eq: ineq2} in Proposition  \ref{prop: transmatrix} to the above inequality yields
\[
\Vinf - V^*_\lam <\frac{|\S|}{C(1-\g)^2}\mathbf{1}.
\]

On the other hand, \eqref{eq: pf_6} can also be written as,
\begin{equation*}
\begin{aligned}
    &(I-\g P^{\pi^*_\lam})(V - V^*_\lam) - \g (P^{\pi} - P^{\pi^*_\lam})V + \lam(\H(\pi) - \H(\pi^*_\lam)) - (r^{\pi} - r^{\pi^*_\lam})= \frac1\b(I-\g P^{\pi})^{-\top}\mathbf{1},
\end{aligned}
\end{equation*}
which is equivalent to,
\begin{equation*}
\begin{aligned}
    &[(I-\g P^{\pi_\lam^*})(V - V_\lam^*)]_s + \sum_{a}(\pi_{sa} - (\pi^*_\lam)_{sa})( - \g\sum_t P_{st}^a V_t -r_{sa}+ \lam\log\pi_{sa} )  \\
    = &\lam\sum_a(\pi^*_\lam)_{sa}(\log(\pi^*_\lam)_{sa} - \log\pi_{sa}) + \frac1\b\l[(I-\g P^{\pi})^{-\top}\mathbf{1}\r]_s.
\end{aligned}
\end{equation*}
Note that the first term of the RHS is the KL divergence of $\pi^*_\lam$ from $\pi$, so it is always
positive. The second term of the RHS is also positive by \eqref{eq: ineq1} in Proposition \ref{prop:
  transmatrix}. Therefore, the RHS of the above equation is larger than $0$, which completes the
proof of \eqref{eq: V lower bd}.
\end{proof}

%\textcolor{blue}{Here}
%Next, we prove in Theorem \ref{thm: reg_2} that the fixed point of the $V$ updates approaches to the regularized optimal value function $V^*_\lam$ as the prefactor $\b$ approaches infinity. The fixed point $\piinf$ for the policy updates is close to the regularized optimal policy $\pi^*_\lam$ when $\lam$ is large and it is close to the true optimal policy $\pi^*$ when $\lam$ is small.

\begin{theorem}\label{thm: reg_2}
For any $\e>0$, if $\b>\frac{2|\S|}{\eps(1-\g)^2}$, then the distance between the fixed point
$(\Vinf,\piinf)$ of Algorithm \ref{algo: V GD} with $\lam >0$ and the maximum and maximizer
$(V^*_\lam,\pi^*_\lam)$ of \eqref{eq: goal} can be bounded by
\[
|\Vinf_s - (V^*_\lam)_s| < \frac\eps2, \text{ for }\forall s\in\S, \quad D_{\text{KL}}(\pi^*_\lam||\piinf) \leq \frac{\e\g}{\lam}.
\]
If one further has $\eps < \frac13g(V^*,\pi^*)$
and $\lam < \frac{\eps(1-\g)}{2\log(|\A|)}$, then the fixed point $\piinf$ is close to the non-regularized
optimal policy in the sense that
\[
D_{\text{KL}}(\pi^*||\piinf) = \log\l(1+\sum_{a\neq a_s}\exp\l(-\frac{\g}{\lam}g_{sa}\r)\r),
\]
where $g(V^*,\pi^*)$ is the same value defined in Theorem \ref{thm: non-reg}, and 
\[
g_{sa} = r_{sa_s} + \g\sum_t P^{a_s}_{st}\Vinf_t - \l(r_{sa} + \g\sum_t P^a_{st}\Vinf_t \r)\l\{\begin{aligned}
=0, \quad a = a_s;\\
>0, \quad a \neq a_s,
\end{aligned}\r. \quad a_s = \argmax_a P^a_{st}\Vinf_t.
\]
\end{theorem}

\begin{remark}
From the above theorem, one can see that as $\b$ approaches infinity, the fixed point $\Vinf$
approaches the regularized optimal value function $V^*_\lam$. However, when $\lam$ is small, the
difference between the fixed point $\piinf$ and the regularized optimal policy $\pi^*_\lam$ could be amplified by $\frac{1}{\lam}$. On the other
hand, by Taylor expansion, the difference between $\piinf$ and the non-regularized optimal policy
$\pi^*$ can be approximated by
\[
D_{\text{KL}}(\pi^*||\piinf) \approx \sum_{a\neq a_s}\exp\l(-\frac{\g}{\lam}g_{sa} \r),
\]
which is close to $0$ when $\lam$ is small. 
\end{remark}

\begin{proof}
  The fixed point of the policy updates satisfies $G_{\th_{sa}}(\Vinf,\piinf) = 0$, where
  $G_{\th_{sa}}$ is defined in \eqref{eq: updates}. That is,
  \begin{equation}\label{eq:gth0}
  \b \rho_s\ell_s\l[-\g \sum_{t} P_{st}^a V_t - r_{sa} +\lam\log\pi_{sa}\r] + c_s = 0.  
  \end{equation}
  It is equivalent to
  \[
  \piinf_{sa}\propto \exp \l(\frac{1}\lam\l(r_{sa} + \g\sum_t P^a_{st}\Vinf_t\r)\r).
  \]
  Let 
  \begin{equation}\label{def of as}
      a_s = \argmax_a\l(r_{sa} +\g\sum_t P^a_{st}\Vinf_t\r),
  \end{equation}
  then $\piinf$ can be written as 
  \begin{align}\label{eq: piinf}
    &\piinf_{sa}\propto \exp\l(-\frac{\g }{\lam} g_{sa}\r),
  \end{align}
  where 
  \[
  g_{sa} =  r_{sa_s} + \g\sum_t P^{a_s}_{st}\Vinf_t - \l(r_{sa} + \g\sum_t P^a_{st}\Vinf_t \r) \l\{\begin{aligned}
  =0, \quad a = a_s;\\
  >0, \quad a \neq a_s.
  \end{aligned}\r.
  \]
  On the other hand, by the equality \eqref{eq:gth0}, one has $ \g\sum_t P_{st}^a \Vinf_t + r_{sa} -
  \lam\log\piinf_{sa} = f_s$, where $f_s$ is a value independent of $a$. Inserting the above $\piinf$ into
  the $\pi$ in \eqref{eq: V lower bd} gives,
  \[
  \begin{aligned}
    [(I-\g P^{\pi_\lam^*})(\Vinf - V^*_\lam)]_s    \geq& f_s \sum_a(\piinf_{sa} - (\pi^*_\lam)_{sa})     =0,
  \end{aligned}
  \]
  where the last equality is due to $\sum_a\piinf_{sa} = \sum_a(\pi^*_\lam)_{sa} = 1$.  Therefore,
  by \eqref{eq: ineq2} in Proposition \ref{prop: transmatrix}, one has $ \Vinf - V^*_\lam > \mathbf{0}$. Combining it with
  Lemma \ref{lemma: reg Vinf} implies
  \begin{equation}\label{eq: pf_7}
    \mathbf{0} < \Vinf - V^*_\lam < \frac{|\S|}{\b(1-\g)^2}\mathbf{1}.
  \end{equation}
  
  On the other hand, $(\pi^*_\lam)_{sa}$ can be represented by
\[
    (\pi^*_\lam)_{sa}\propto \exp \l(\frac{1}\lam\l( r_{sa} + \g\sum_t P^a_{st}(V^*_\lam)_t\r)\r).
\] 
Since for $\forall s\in\S$,
\[
\begin{aligned}
    &\max_a\lv \frac{1}\lam\l( r_{sa} + \g\sum_t P^a_{st}(V^*_\lam)_t\r) - \frac{1}\lam\l( r_{sa} + \g\sum_t P^a_{st}\Vinf_t\r) \rv \\
    = &\max_a\lv \frac{\g}\lam\sum_t P^a_{st} ((V^*_\lam)_t - \Vinf_t) \rv 
    <\frac{\g|\S|}{\lam\b(1-\g)^2},    
\end{aligned}
\]
by Lemma \ref{lemma: ineq}, one has
\begin{equation}\label{eq: pf_8}
    D_{\text{KL}}((\pi^*_\lam)_s | \piinf_s) \leq \frac{2\g|\S|}{\lam\b(1-\g)^2}.
\end{equation}
To sum up, if $\b > \frac{2|S|}{\e(1-\g)^2}$, then by \eqref{eq: pf_7} and \eqref{eq: pf_8}
\[
  \mathbf{0}<\Vinf - V^*_\lam < \frac{\e}{2}\mathbf{1}, \quad D_{\text{KL}}(\pi^*_\lam | \piinf) <\frac{\e\g}{\lam},
\]
which completes the proof for the first part of the lemma.

For the second part, note that
\[
    V^*_\lam = V^{\pi^*_\lam}_\lam \geq V^{\pi^*}_\lam = (I-\g P^{\pi^*})^{-1} r^\pi +(I-\g P^{\pi^*})^{-1} ( -  \lam H(\pi^*) )\geq V^{\pi^*} = V^*;
\]
\[
    V^*_\lam = V^{\pi^*_\lam} +(I-\g P^{\pi_\lam^*})^{-1} ( -  \lam H(\pi_\lam^*) )  \leq V^{\pi^*} + \frac{1}{1-\g}\max_s (- \lam\H(\pi^*_\lam)_s) \leq V^* + \frac{\lam}{1-\g}\log(|\A|),
\]
where one applies \eqref{eq: ineq1} in Proposition \ref{prop: transmatrix} to the second inequality on the first equation and the first inequality on the second equation. Hence, one has
\[
    \mathbf{0} < V^*_\lam - V^* \leq \frac{\lam}{1-\g}\log(|\A|)\mathbf{1}.
\]
Combining it with the inequality \eqref{eq: pf_7},  one has
\[
|\Vinf_t - V^*_t| \leq |\Vinf_t - (V^*_\lam)_t| + |(V^*_\lam)_t - V^*_t| < \frac{|\S|}{\b(1-\g)^2} + \frac{\lam}{1-\g}\log(|\A|).
\]
Therefore, when $\b>\frac{2|\S|}{\eps(1-\g)^2}$ and $\lam < \frac{\eps(1-\g)}{2\log(|\A|)}$, then
$|\Vinf_t - V^*_t| < \eps$ for all $t\in\S$. As proved in Lemma \ref{lemma: as_asstar}, when
$|\Vinf_t-V^*_t| < \eps = \frac13g(V^*,\pi^*)$ for all $t\in\S$, then $a_s = a^*_s$ with $a^*_s$ defined in Lemma \ref{lemma:
  optimal policy} and $a_s$ defined in \eqref{def of as}. By the definition of $\pi^*$ in Lemma
\ref{lemma: optimal policy} and $\piinf$ in \eqref{eq: piinf}, one has
\[
D_{\text{KL}}(\pi^*||\piinf) = \log\l(\frac{1}{\piinf_{s a_s^*}}\r) = \log\l(1+\sum_{a\neq a_s}\exp\l(-\frac{\g}{\lam}g_{sa}\r)\r),
\]
which completes the proof for the second part of the lemma.

%The fixed point of $V$ satisfies $G_V = 0$, which gives \[\Vinf = (I - \g P^{\pi})^{-1}(r -
%\lam \H(\pi)) + \frac1C(I - \g P^{\pi})^{-1}(I - \g P^{\pi})^{-\top}\mathbf{1}. \]

\end{proof}

%==============================================================================================================
\section{Numerical Experiments}\label{sec: numerics}

This section studies the performance of the model-based and model-free algorithms numerically
(Algorithms \ref{algo: V GD} and \ref{algo: Q}). Two different MDPs, one with states embedded in the
1D space and another with states in the 2D space, are used as testing examples.  The numerical
experiments demonstrate that both non-regularized ($\lam=0$) and regularized ($\lam>0$) versions of
the proposed algorithm converge to policies close to the non-regularized optimal policy $\pi^*$.  In addition, the algorithm combined with the BFF idea solves the
double sampling problem. A comparison between the flipping method and the natural policy gradient
(NPG) method is also provided to demonstrate that the flipping method outperforms the NPG
method.

%=======================
\subsection {Example 1}

Consider an MDP with a discrete state space $\S = \l\{s_k = \frac{2\pi
  k}{n}\r\}_{k=0}^{n-1}$. The transition dynamics is given by
\begin{equation}\label{MDP}
\begin{aligned}
  &\tilde{s}_{t+1} = \text{mod}\l(s_t + \frac{2\pi}{n} a_t + \sigma Z_t,\, n\r) , \quad s_{t+1} =
  \l\{
  \begin{aligned}
    &\argmin_{i\in Z, i \in [0,n-1]} \lv \tilde{s}_{t+1} - i \rv, \quad \text{if } \tilde{s}_{t+1} \in [0,n-1/2),\\
      &0,\quad \text{if } \tilde{s}_{t+1}\in [n-1/2, n),
  \end{aligned}
  \r.    
\end{aligned}
\end{equation}
where $a_t\in\A = \{\pm 1\}$ and $Z_t\sim N(0, 1)$ follows the normal distribution. The reward
function $r(s)=1+\sin(s)$. In Figures \ref{fig: fig00}-\ref{fig: fig2}, $\sigma = 0$, i.e. the dynamics is deterministic given the current state and action. In Figure \ref{fig: fig3}, $\s \neq 0$ and
hence given the current state and action the next state is stochastic.

{\bf Results of Algorithm \ref{algo: V GD}.} Here we assume 
that the transition dynamics is known. $V(s)$ is represented in the
tabular form and $\pi(s,a)$ is parameterized with the soft-max function. Since it is shown in Figure
\ref{fig: figh0} that the vanilla gradient descent results in increasing error in the initial stage, only the
clipping and flipping methods are tested here. Both the non-regularized ($\lam =
0$) and regularized ($\lam = 0.1$) method are tested.  The error $\pi_k - \pi^*$ in the $L^1$ norm is
shown in Figure \ref{fig: fig00}. In order to demonstrate the stability of the algorithm, $100$
simulations with different initializations are run for each case and the mean of $100$ simulations
is plotted in a darker color. The learning rates $\eta_V$ and $\eta_\pi$ are both set to be
$1/(4\b)$ for all cases.

First, for both the regularized and non-regularized method, the difference between $\pi_k$ and the true optimal policy $\pi^*$ approaches to $0$. Second, the prefactors that make the number of states $n = 5, 55,
105$ converge are $\b = 10, 100, 1000$, respectively. As the number of states increases, the
prefactor increases as expected, which is consistent with what we demonstrated in Theorems \ref{thm:
  non-reg} and \ref{thm: reg_2}. In addition, one finds that the flipping method decays
consistently, while the clipping method decays slowly at first and then matches the rate of the
flipping method.

\begin{figure}[h!]
    \centering
    \includegraphics[width=\linewidth]{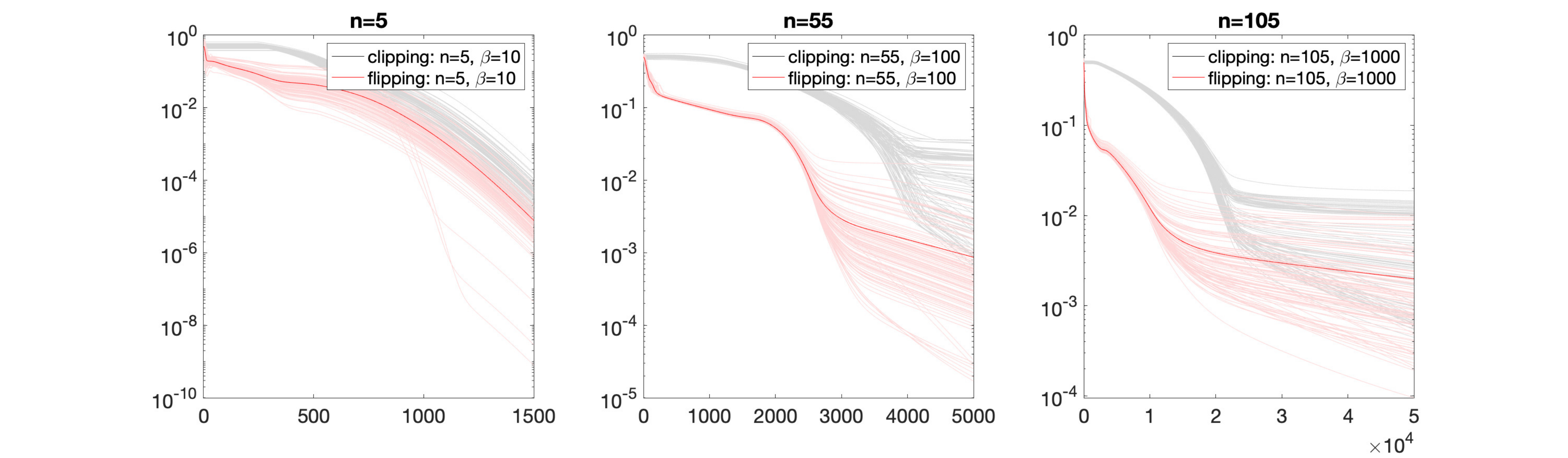}
    \includegraphics[width=\linewidth]{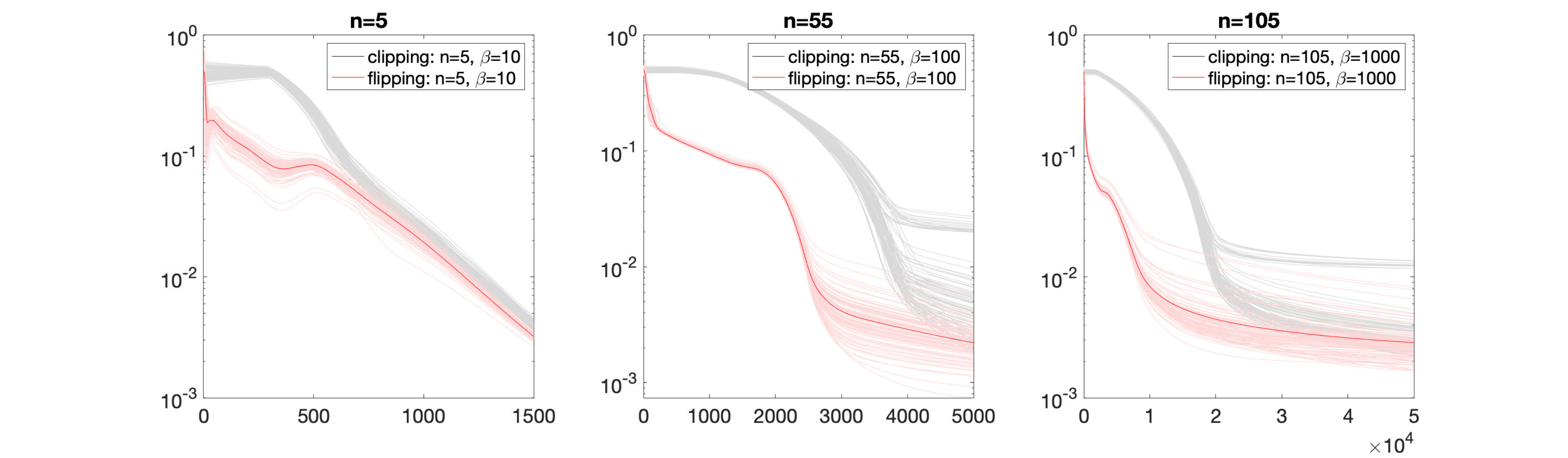}
    \caption{The plots show the error $\ll \pi_k - \pi^*\rl_{L^1}$ from Algorithm \ref{algo: V GD} for the size of the state
      space $n = 5, 55, 105$ from left to right. The first row is for the non-regularized method, while the second row is for the regularized method with $\lam = 0.1$. We run $100$
      simulations for each case and plot the mean in black and red. }
    \label{fig: fig00}
\end{figure}

%so both non-regularized MDP and regularized MDP converges to the non-regularized optimal policy. In
%addition, one can see that as the state space becomes larger, the penalty constant $C$ that makes
%it converge becomes larger. With the same size of the state space, the clipping method converges
%slowly at first, and then it converges at the same rate as the flipping method. The flipping method
%converges faster than the clipping method.

{\bf Results of Algorithm \ref{algo: Q} with different prefactors.}
Here we assume that the transition dynamics is unknown.  $Q(s,a)$ is represented by the tabular form and $\pi(s,a)$ is parameterized by the soft-max
function.  Note that in this example, given $s$ and $a$, the transition dynamics is
deterministic. Therefore, one only needs to duplicate the first sample for the next state to the
second sample, namely, letting $s'_{t+1} = s_{t+1}$ in Algorithm \ref{algo: Q}. The off-policy
$\pi_b(a|s) = 1/2$ for $\forall s\in\S$ is used to generate the trajectory
$\{s_t,a_t,r_t\}_{t=0}^T$.  The error $\pi^k - \pi^*$ in the $L^1$ norm is shown in Figure \ref{fig:
  fig0}. In order to show the stability of the algorithm, $100$ simulations (with different
off-policy trajectories and different parameter initializations) are run for each case and the mean
of $100$ simulations is plotted in a darker color. To encourage exploration, we set $\lam=0.1$. The
learning rate $\eta_\pi = 4/C$, $\eta_Q = 30/C$, and the batch size $M = 1000$.

\begin{figure}[h!]
    \centering
    \includegraphics[width=0.8\linewidth]{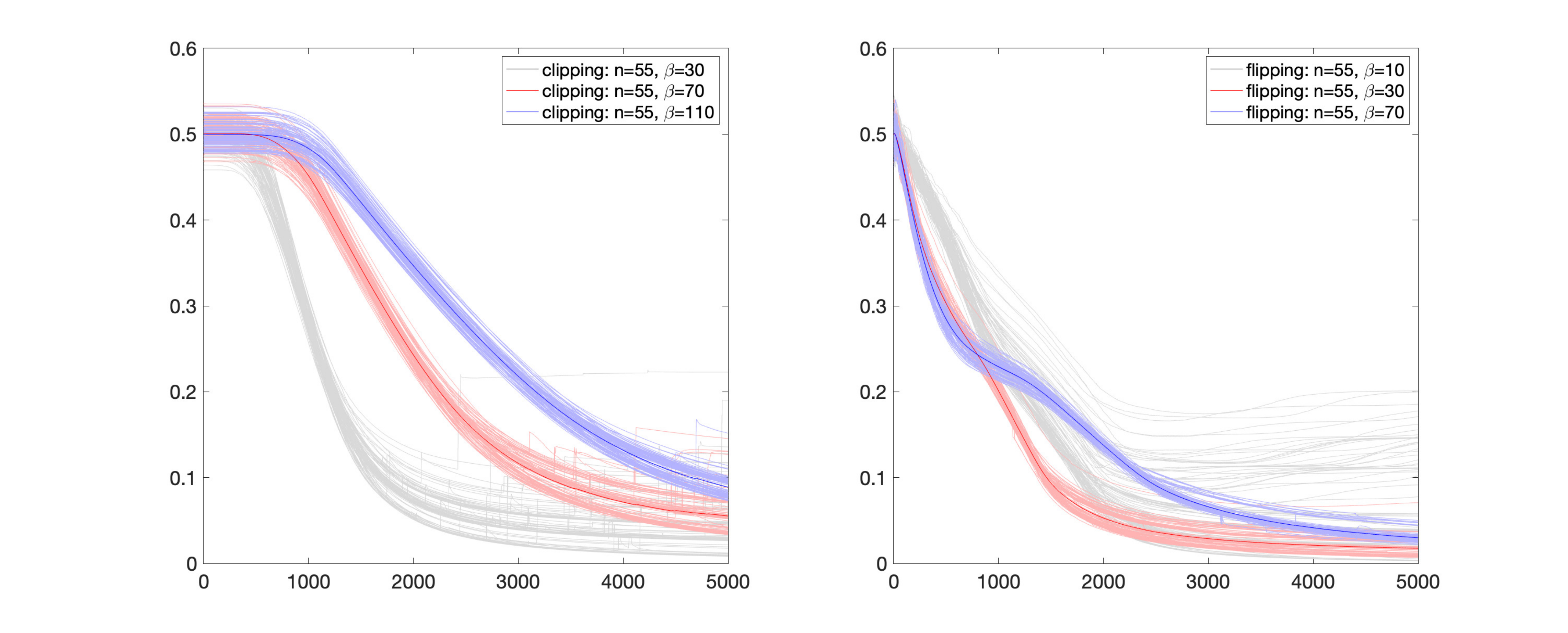}
    \caption{The plots show the error $\ll \pi_k - \pi^*\rl_{L^1}$ from Algorithm \ref{algo: Q} with $\lam = 0.1$ for
      different prefactors $\b$. The left one is the results of the clipping method , while the
      right one is the results of the flipping method. The grey, pink and blue lines represent $100$
      simulations with $\b$ from small to large. The mean of each case is plotted in black, red and
      dark blue. The error is for the non-regularized optimal policy $\pi^*$.
      %One can see that when the penalty constant $C$ is small, there are more simulations converge
      %to points far from the optimal policy; as $C$ becomes larger, all the results will converge
      %to points closer to the optimal policy.
    }
    \label{fig: fig0}
\end{figure}

Figure \ref{fig: fig0} shows that, for both clipping and flipping methods, the probability of
reaching the optimal policy $\pi^*$ becomes larger as the prefactor $\b$ grows. Furthermore,
clipping still has several simulations diverge with $\b=70$, while all the simulations for 
flipping converge with $\b = 70$. Therefore, flipping requires a smaller $\b$ to be convergent
compared with clipping.

\begin{figure}[h!]
    \centering
    \includegraphics[width=\linewidth]{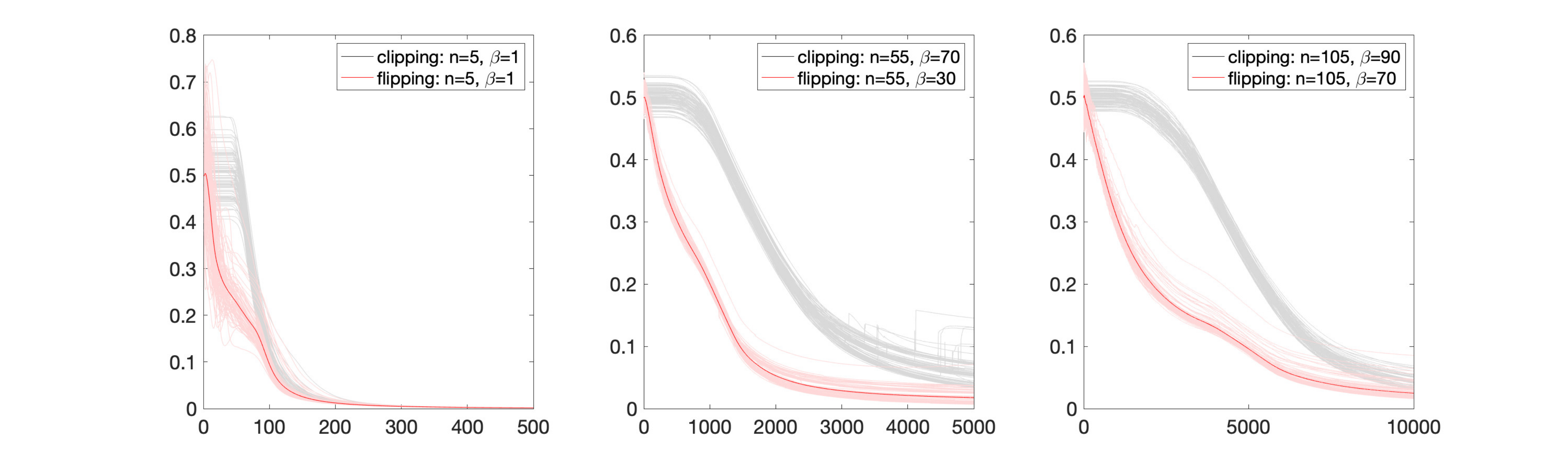}
    \caption{The plots show the error $\ll \pi_k - \pi^*\rl_{L^1}$ from Algorithm \ref{algo: Q} with $\lam = 0.1$ for the
      size of the state space $n = 5, 55, 105$ from left to right. The grey and pink lines represent
      $100$ simulations for clipping and flipping methods, respectively. The mean of each case is
      plotted in black and red. 
      %One can see that as the state space becomes larger, the penalty constant $C$ that makes it
      %converge becomes larger. With the same size of the state space, the clipping method converges
      %slowly at first, and then it converges at the same rate as the flipping method. The flipping
      %method converges faster than the clipping method.
    }
    \label{fig: fig1}
\end{figure}

Figure \ref{fig: fig1} compares the convergence curves of clipping and flipping. Similar to Figure
\ref{fig: fig00}, the error for the clipping method decays slowly at first, while the error from
flipping decays consistently.  Comparing Figure \ref{fig: fig1} with Figure \ref{fig: fig00}, one can
see that the stochastic algorithm converges in fewer steps. The reason is that one can set the
prefactor $\b$ smaller and the learning rate larger to encourage stochasticity. 

{\bf Comparison with other methods.} Figure \ref{fig: fig2} compares the flipping method with the
natural policy gradient method (NPG) given by,
\begin{equation}\label{eq: NPG1}
  \pi^{k+1}_{sa} = \l(\pi^k_{sa}\r)^{1-\frac{\lam\eta_{\pi}}{1-\g}}\exp\l(\frac{\eta Q^{\pi^k}(s,a)}{1-\g}\r),
\end{equation}
where $Q^{\pi^k}(s,a)$ is estimated by solving the residual Bellman minimization problem
\[
\min_Q \underset{(s,a)\sim\rho}{\E}\l(Q(s,a) - \underset{s'\sim P^a(\cdot|s,a)}{\E}\l[\g\sum_a (
  Q(s',a)- \lam \log\pi(s',a))\pi(s',a) |s,a\r] - r(s,a)\r)^2.
\]
The algorithm for $Q^{\pi^k}$ updates $Q^j$ with initialization $Q^0 = Q^{\pi^{k-1}}$ and stops when
$\sum_{s,a}(Q^{j}(s,a) - Q^{j-1}(s,a))^2/n < \eps$:
\begin{equation}\label{eq: NPG}
  \begin{aligned}
    &w_t = Q^j(s_t,a_t) - r_t - \g \l(V^j(s_{t+1}) - \lam H^k(s_{t+1}) \r),\\
    &G_t(s,a) = w_t\mathds{1}_{s = s_t, a = a_t} - \g \pi(s_{t+1},a)w_t\mathds{1}_{s = s_{t+1}'},\\
    &Q^{j+1} = Q^j - \eta_Q \sum_{t=(j-1)M+1}^{jM}G_t;\quad V^{j+1}(s) = \sum_{a}Q^{j+1}(s,a)\pi^k(s,a),
  \end{aligned}
\end{equation}
where $H^k(s) = \sum_a\pi^k(s,a) \log(\pi^k(s,a))$. The batch size $M = 1000$ and regularization
constant $\lam = 0.1$ are the same for both methods. For the NPG method, we set $\eps=2\times10^{-4}, \eta_Q=4$, $\eta_\pi=0.1$ for $n=55$ and $\eps=2\times10^{-2},
\eta_Q=30$, $\eta_\pi=0.1$ for $n=55$. For the flipping method, we set $\b = 1$, $\eta_Q = 2/C,
\eta_\pi = 1/C$ for $n = 5$ and $\b = 80, (\eta_Q)_k=(\eta_\pi)_k = \min(0.999^k\frac{30}{C},
\frac{20}{C})$ for $n = 55$.

\begin{figure}[h!]
  \centering
  \includegraphics[width=0.8\linewidth]{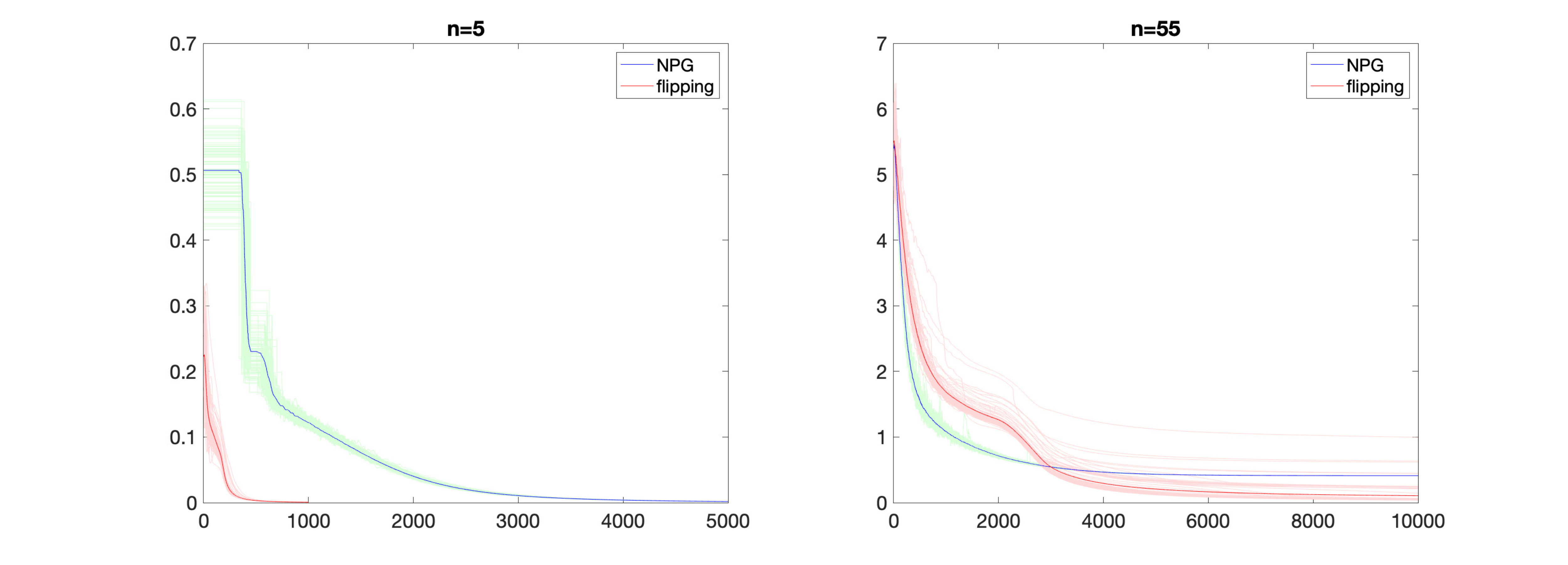}
  \caption{The plots show the comparison of the error $\ll \pi_k - \pi^*\rl_{L^1}$ between the flipping method and the NPG method \eqref{eq:
      NPG1} - \eqref{eq: NPG} for the size of the state space $n=5$ and $n= 55$. Both use the
    regularized objective function, i.e., $\lam = 0.1$. The green and pink lines represent $100$
    simulations for the NPG method and the flipping method, respectively. The mean of each case is
    plotted in blue and red. 
    %One can see that for $n=5$, our method converges much faster than the NPG method; while for
    %$n=55$, with high probability, our method converges to the optimal policy with much smaller
    %error than the NPG method.
    }
    \label{fig: fig2}
\end{figure}

For $n = 5$, though both methods converge to the optimal policy $\pi^*$, the flipping method
converges faster than the NPG method. For $n=55$, NPG converges to the regularized optimal policy
$\pi^*_\lam$, while our method converges to a policy close to the true optimal policy $\pi^*$ with a high
probability. %The convergence rates are similar. In general, our method outperforms the NPG method.

{\bf Results of  Algorithm \ref{algo: Q} with BFF.}
%\LY{confusing. do you mean the transition is deterministic?} In the previous experiments, the next state can be uniquely determined by the current state and action.

Here we assume the transition dynamics is stochastic given the current state and action.  We set $\sigma = 1$ for $n=5$, $\sigma = 0.5$ for $n = 55$ and $\sigma = 0.1$ for $n = 105$. Unlike Figure \ref{fig: fig1} -
\ref{fig: fig2}, BFF is used to approximate the second independent sample for the next state in Figure \ref{fig:
  fig3}. Other than that, the setting remains the same as Figure \ref{fig: fig1}. One can see that
BFF provides a good approximation for the gradient. The approximation error of the flipping method decays quickly 
for all three different cases.

\begin{figure}[h!]
    \centering
    \includegraphics[width=\linewidth]{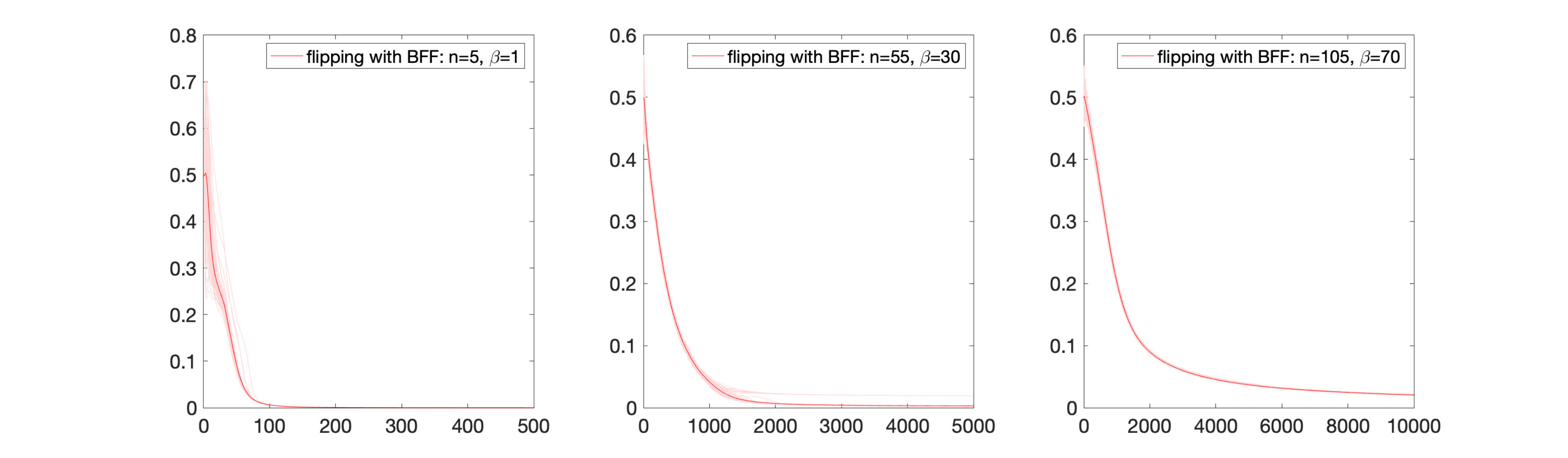}
    \caption{The plots show the error $\ll \pi_k - \pi^*\rl_{L^1}$ from the flipping method of Algorithm \ref{algo: Q} with BFF for the size of the
      state space $n = 5, 55, 105$ from left to right. The pink lines represent $100$ simulations,
      and their mean is plotted in red.}
    \label{fig: fig3}
\end{figure}

%=======================
\subsection{Example 2}

Consider another MDP with a discrete state space $\S = \{s_{ij} \}_{i,j=0}^{i = n_1-1, j = n_2-1}$,
where $s_{ij} = (i,j)$ is a  two-dimensional vector. The transition dynamics is given by
\[
\begin{aligned}
    &\tilde{s}_{t+1} \gets s_t + (1+ \sigma Z_t )a_t, \\
    &(\tilde{s}_{t+1})_k \gets \text{mod}\l((\tilde{s}_{t+1})_k, n_k\r), \quad k = 1,2,\\
    &(s_{t+1})_k = 
    \l\{
    \begin{aligned}
        &\argmin_{i\in Z, i \in [0,n_k-1]} \lv (\tilde{s}_{t+1})_k - i \rv, \quad \text{if } (\tilde{s}_{t+1})_k \in [0,n_k-1/2),\\
        &0,\quad \text{if } (\tilde{s}_{t+1})_k\in [n_k-1/2, n_k),
    \end{aligned}
    \r.
\end{aligned}
\]
where $a_t\in\A = \{(\pm 1,0), (0,\pm1)\}$ and $Z_t\sim N(0, 1)$. $n_1 = n_2 = 7$, the reward is set
to be $r(s_{ij})=2 + \sin\l(\frac{2\pi i}{n_1}\r) + \cos\l(\frac{2\pi j}{n_2}\r)$, and the noise
$\s$ is set to be $0.1$.

\begin{figure}[h!]
    \centering
    \includegraphics[width=0.8\linewidth]{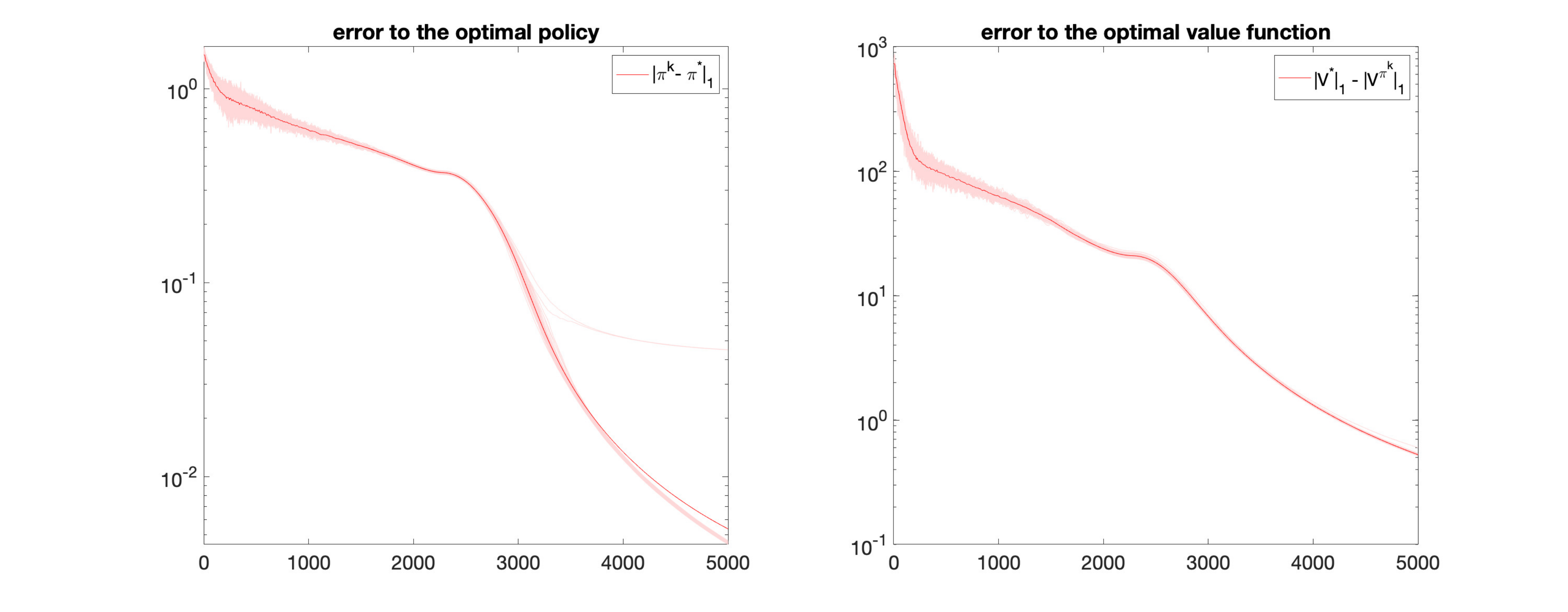}
    \caption{The above plots show the convergence of the flipping method to the optimal policy $\pi^*$ and
      optimal value function $V^*$. The pink lines represent $100$ simulations that correspond to different
      off-policy trajectories and initializations of the parameters. The mean is plotted in red.  }
    \label{fig: fig4}
\end{figure}

The result is plotted in Figure \ref{fig: fig4}. We set $\lam = 0$ for the non-regularized objective
function and use BFF to approximate the second independent sampling for the next state. The
prefactor and learning rate are set to be $\b = 30$ and $\eta_Q=\eta_\pi=\frac{30}{\b}$. The error
is plotted out in the $L_1$ norm. $98\%$ of the simulations converge to the true optimal policy
$\pi^*$. Note that the value function in the right plot of Figure \ref{fig: fig4} is the value 
function $V^{\pi^k}$ under the policy $\pi^k$, which is different from the $V^k$ in the algorithm. 
It shows that the policy indeed consistently maximizes the value function $V^{\pi}$.

\bibliographystyle{plain}

%===================================================================
\newpage
\appendix
\section*{Appendices}
\addcontentsline{toc}{section}{Appendices}
\renewcommand{\thesubsection}{\Alph{subsection}}

%===========

\subsection{SGD Algorithm for V}\label{appendix: SGD_V}
\setcounter{equation}{0}
\setcounter{theorem}{0}
\renewcommand\theequation{A.\arabic{equation}}
\renewcommand\thetheorem{A.\arabic{theorem}}
Given a trajectory $\{(s_t,a_t,r_t)_{t=0}^T\}$, the unbiased stochastic estimate for the gradient of \eqref{eq: objobj} is
\[
\begin{aligned}
    &(G_V)_t = -\nb_\o V^k(s_t) + \b L_t \l(\nb_\o V^k(s_t) - \g \frac{\pi^k(s_t,a_t')}{\pi_b(s_t,a'_t)}\nb_\o V^k(s_{t+1}')\r);\\
    &(G_\pi^i)_t = \b \h{h}^{(i)}(L_t)\l(-\g \frac{\pi^k(s_t,a_t')}{\pi_b(s_t,a_t')}V^k(s'_{t+1})\nb_\th [\log\pi^k(s_t,a_t')] + \lam \sum_a \nb_\th\pi^k(s_t,a)(\log\pi^k(s_t,a) +1) \r);
\end{aligned}
\]
where $V^k(s) = V(s,\o_k)$, $\pi^k(s,a) = \pi(s,a,\th_k)$ and $\h{h}^{(i)}$ is defined in \eqref{def of
  hp}. Here $L_t$ is the estimates for the Bellman residual,
\[
L_t = V^k(s_t) - r_t - \g\frac{\pi^k(s_t,a_t)}{\pi_b(s_t,a_t)} V^k(s_{t+1}) + \lam
\sum_a\pi^k(s_t,a)\log\pi^k(s_t,a).
\]
$a'_t$ is a sample from $\pi_b(s_t,a)$ that is uncorrelated with $a_t$, and $s'_{t+1}$ is a sample for the next state
when action $a'_t$ is taken at state $s_t$. Here we use the BFF algorithm proposed in
\cite{zhu2020borrowing} to approximate this two samples.
\[
a'_t = a_{t+1}, \quad s'_{t+1} = s_t + (s_{t+2} - s_{t+1}). 
\]
 The stochastic algorithm for the V-formulation is summarized in Algorithm \ref{algo: V}.
%\[ \o_{k+1} = \o_k - \frac{\eta_V}M\sum_{t = (k-1)M+1}^{kM} G^t_V,\quad \th_{k+1} = \th_k -
%\frac{\eta_\pi}M\sum_{t = (k-1)M+1}^{kM} (G^i_\pi)^t,\]

\begin{algorithm}
  \caption{V-formulation}
  \label{algo: V}
  \begin{algorithmic}[1]
    \REQUIRE $\eta_V, \eta_\pi$: prefactor; $\b$: penalty constant; $M$: batch size; \\  $V(s,\o),\pi(s,a,\th)$: parametrized approximation of $V(s),\pi(s,a)$; \\
    $\{s_t,a_t,r_t\}_{t=0}^T$: trajectory generated from off-policy $\pi_b$;
    \STATE Random initialization of $\th_0,\o_0$, $k=0$
    \WHILE{$\o,\th$ do not converge}
        \STATE $j\gets 0$, $k\gets k+1$
        \FOR{$t=(k-1)M+1,\cdots,kM$}
        \STATE $s_j = s_t$
        \STATE $L_j = V(s_t,\o) - r_t - \g \tau(s_t,a_t)V(s_{t+1},\o) + \lam\H(s_t) $
        \STATE $s'_{t+1} \gets s_t + (s_{t+2} - s_{t+1})$; \quad $a'_t \gets a_{t+1}$
        \STATE $G_V^j = -\nb_\o V(s_t,\o) + \b L_t(\nb_\o V(s_t,\o) - \g\tau(s_t,a'_t) \nb_\o V(s'_{t+1},\o))$
        \STATE $G_\pi^j = \b \l(-\g\tau(s_t,a'_t)V(s'_{t+1},\o)\nb_\th\log\pi(s_t,a_t,\th) + \lam\sum_a(\log\pi(s_t,a,\th)+1)\nb_\th\pi(s_t,a,\th) \r)$
        \STATE $j \gets j+1$
        \ENDFOR
        \STATE $\ds G_V \gets \frac1M\sum_{j=1}^MG_V^j $; \quad $\o \gets \o - \eta_V G_V$
        \STATE $\h\ell_s \gets \sum_{s_j = s} L_j$
        \STATE $\ds G^{(i)}_\pi\gets \frac1M\sum_{j=1}^M\h{h}^{(i)}(L_j)G_\pi^j $; \quad  $\th \gets \th - \eta_\pi G^{(i)}_\pi$, \quad where $\h{h}^{(i)}$ is defined in \eqref{def of hp}
        \STATE $\tau(s,a) \gets \frac{\pi(s,a,\th)}{\pi_b(s,a,\th)}$; $\H(s) \gets \sum_a \pi(s,a,\th)\log\pi(s,a,\th)$
    \ENDWHILE
\end{algorithmic}
\end{algorithm}

\subsection{Fixed point of Algorithm \ref{algo: V GD} with $\lam = 0$ is not stochastic policy}\label{appendix: gthneq0}
\setcounter{equation}{0}
\setcounter{theorem}{0}
\renewcommand\theequation{B.\arabic{equation}}
\renewcommand\thetheorem{B.\arabic{theorem}}

\begin{lemma}\label{lemma: gthneq0}
Assume $|\A|>2$, the null space of $P^a - P^{a'}$ is the linear space spanned by ${\bf 1}$ for all $a\neq a'$, and the reward is not a constant, i.e., $r_{sa} \not\equiv r$. When $\beta$ is sufficiently large, then $(G_V,G_\th) \neq (0,0)$.

%assume there exists $a\neq a'$, such that the projection of $r^a - r^{a'}$ onto the null space of $P^a - P^{a'}$ is not ${\bf 0}$, then it is impossible that $(G_V,G_\th) = (0,0)$. 
%Assume the null space of $P^a - P^{a'}$ is the the linear space spanned by ${\bf 1}$ for all $a\neq a'$ and $r^{a} = r$ for $\forall a\in\A$. 
\end{lemma}

%Since all the transition matrix have a common eigenvector $\mathbf{1}$,  the assumption means that there is only one common eigenvector for all the transition matrix. 

\begin{proof}
Assume that $G_\th = 0$, then it gives
\[
\g P^aV + r^a = c , \quad \forall a.
\]
where $r^a = (r_{sa})_{s\in\S}$ is an $|\S|$-dimensional vector and $c$ is a constant vector. This is equivalent to,
\begin{equation}\label{eq: pf_4}
   \g (P^a - P^{a'})V = r^{a'} - r^{a} , \quad \forall a \neq a'. 
\end{equation}
Note that if there exists three different actions $a,a',a''$, such that $r^a - r^{a'} = r^a - r^{a''} = c\neq{\bf 0}$, then $r^a = r^{a''} = {\bf 0}$. Therefore, when $|\A|\geq3$, the value of $r^{a'} - r^{a}$ can be separated into two different cases. The first case is that 
\begin{equation}\label{eq: pf_51}
    \text{there exists three different actions } a,a',a''\in\A \text{ such that } r^a - r^{a'} \neq r^a - r^{a''} \neq {\bf 0}. 
\end{equation}
The second case is that 
\begin{equation}\label{eq: pf_52}
    \text{there exists two different actions } a,a'\in\A \text{ such that } r^a = r^{a'} = r,
\end{equation}
where $r$ is a constant vector.

Let us consider the first case where $r^a - r^{a'} \neq r^a - r^{a''}$ and both $r^a - r^{a'}$ and $r^a - r^{a''}$ are not equal to ${\bf 0}$. Let $\mN$ be the null space of $(P^a - P^{a'})$ for $\forall a\neq a'$, which is a linear space spanned by ${\bf 1}$. If the projection of $r^a - r^{a'}$ onto  $\mN$ is not equal to ${\bf 0}$, then there is no solution for $V$ in \eqref{eq: pf_4}. If the projection of $r^a - r^{a'}$ and $r^a - r^{a''}$ onto the null space $\mN$ are both equal to ${\bf 0}$, then there does not exist a vector $V$, such that $\g(P^a - P^{a'})V = r^{a'} - r^{a}$ and $\g(P^a - P^{a''})V = r^{a''} - r^{a}$. Therefore, there is no solution for \eqref{eq: pf_4}. To sum up, $G_\th \neq 0$ for the first case \eqref{eq: pf_51}.

Next, let us consider the second case where $r^a = r^{a'} = r$, then $V = c_1\mathbf{1}$ is the only solution to \eqref{eq: pf_4}. Given that $P^\pi$
is transition matrix, $P^\pi \mathbf{1} = \mathbf{1}$. Plugging it into $G_V = \mathbf{0}$ yields,
\begin{equation}\label{eq: pf_1}
-\rho + \b(I-\g P^\pi)^\top\l[\l((1-\g)c_1\mathbf{1} - r\r)\odot\rho\r] = \mathbf{0}.    
\end{equation}
Multiplying $\mathbf{1}^\top$ to \eqref{eq: pf_1} gives
\[
(1-\g)c_1 = \bar{r} + \frac{1}{\b(1-\g)},
\]
where $\bar{r} = \sum_{s} r_s\rho_s$. Plugging it back to \eqref{eq: pf_1} leads  to
\[
\bar{r} - r + \frac{1}{\b(1-\g)} = \frac{1}{\b}(I - \g P^\pi)^{-\top}\mathbf{1}.
\]
When $\beta>\frac{1}{(1-\g)(\max_sr_s - \bar{r})}$, then at least one element of the LHS is negative. However, the RHS is
always positive by Propsition \ref{prop: transmatrix}, which gives contradiction. Therefore, $(G_V, G_\th) \neq (0,0)$ for the second case \eqref{eq: pf_52}. 
\end{proof}

\subsection{Proof of Proposition \ref{prop: transmatrix}}\label{appendix: proof of prop}
\setcounter{equation}{0}
\setcounter{theorem}{0}
\renewcommand\theequation{C.\arabic{equation}}
\renewcommand\thetheorem{C.\arabic{theorem}}

\begin{proof}
Let $x = (I-\g P)^{-1}c$, and assume $x_s = \min_i x_i \leq 0$. The $s$-th component of $(I- \g P)x = c$ is 
\[
c_s = x_s - \g \sum_t P_{st}x_t \leq x_s - \g \sum_t P_{st} x_s = x_s -\g x_s \leq 0,
\] 
which contradicts with the assumption $c_s>0$ for $\forall s$. On the other hand, by letting
$x_{s'}=\max_i x_i$, the $s'$-th component of the $(I-\g P)x = c$ is
\[
(1 - \g)x_{s'}\leq x_{s'} - \g \sum_tP_{s't}x_t = c_s \leq \max_i c_i.
\]
Therefore, 
\[
x_{s'}\leq \frac{\max_i c_i}{1 - \g},
\]
which completes the proof for the first part. 

For $(I- \g P)^\top x = c$, summing over all the components that $x_s \leq 0$ yields
\begin{equation*}
\begin{aligned}
    \sum_sc_s\mathds{1}_{x_s\leq 0} =& \sum_{s} x_s\mathds{1}_{x_s\leq 0} - \g \sum_{s,t}P_{ts}x_t\mathds{1}_{x_s\leq 0} \\
    =& \sum_{s} x_s\mathds{1}_{x_s\leq 0} - \g \sum_{s,t}P_{ts}x_t\mathds{1}_{x_t\leq 0}\mathds{1}_{x_s\leq 0}- \g \sum_{s,t}P_{ts}x_t\mathds{1}_{x_t> 0}\mathds{1}_{x_s\leq 0}\\
    \leq & \sum_{s} x_s\mathds{1}_{x_s\leq 0} + \g \sum_{t}\l(\sum_{s}P_{ts}\mathds{1}_{x_s\leq 0}\r)\l(-x_t\mathds{1}_{x_t\leq 0}\r)\\
    \leq & (1-\g)\sum_{s} x_s\mathds{1}_{x_s\leq 0}  \leq 0.\\
\end{aligned}
\end{equation*}
The first inequality holds because the last term on the second line is always $\leq 0$. The second
inequality is due to $\sum_{s}P_{ts}\mathds{1}_{x_s\leq 0} \leq \sum_{s}P_{ts} =1$ for $\forall t$.
However, the LHS is always strictly larger than $0$, which gives a contradiction. Therefore all
components of $x$ are positive. On the other hand, note that
\[
(1-\g)\sum_s x_s = \mathbf{1}^\top (I-\g P)^\top x = \mathbf{1}^\top c = \sum_s c_s,
\]
and $x_s>0$, therefore, $x_s < \sum_s x_s = \frac{\sum_s c_s}{1-\g}$, which completes the proof for
the second part.

For $(I-\g P )x \leq c\mathbf{1}$, let $x_s = \max_i x_i$, then the $s$-th component of $(I-\g P )x\leq c\bf{1}$ is
\[
c\geq x_s-\g\sum_tP_{st}x_t \geq x_s - \g\sum_tP_{st}x_s = (1-\g)x_s,
\]
which leads to $(1-\g)x_s \leq c$. Therefore, $x\leq x_s \leq \frac{c}{1-\g}\mathbf{1}$.
\end{proof}

\subsection{Proof of Lemma \ref{lemma: ineq}} \label{appendix: ineq}
\setcounter{equation}{0}
\setcounter{theorem}{0}
\renewcommand\theequation{D.\arabic{equation}}
\renewcommand\thetheorem{D.\arabic{theorem}}
\begin{proof}
First note that
\begin{equation}\label{eq: pf_3}
    \log(\pi_a) - \log(\mu_a) = \th_a - \o_a - \l( \log\l(\sum_b\exp(\th_b)\r) - \log\l(\sum_b\exp(\o_b)\r) \r).
\end{equation}
Let $f(x) = \log\l(\sum_a\exp(x_a)\r)$ be a function mapping $x\in\R^{d}$ to $\R$, then 
\[
\nb_{x}f = \frac{\exp(x_a)}{\sum_b\exp(x_a)},
\]
which implies that $\ll \nb_xf(x)\rl_1 = 1$ for $\forall x\in\R^d$.
By the mean value theorem, one has $\forall \th,\o\in\R^d$,
\[
\begin{aligned}
&\lv \log\l(\sum_b\exp(\th_b)\r) - \log\l(\sum_b\exp(\o_b)\r) \rv= \lv f(\th) - f(\o) \rv = \lv \la \th - \o, \nb_xf(x)\ra \rv \\
\leq &\max_a | \th_a - \o_a | \ll \nb_xf(x)\rl_1 = \max_a | \th_a - \o_a |,    
\end{aligned}
\]
where $x$ is a convex combination of $\th$ and $\o$.
Applying the above inequality into \eqref{eq: pf_3} yields
\[
    \log(\pi_a) - \log(\mu_a) \leq 2 \max_a | \th_a - \o_a |.
\]
Therefore, 
\[
    D_{\text{KL}}(\pi|\o) = \sum_a \pi_a (\log(\pi_a) - \log(\o_a)) \leq 2b \sum_a\pi_a = 2b.
\]
\end{proof}

\end{document}